\documentclass[11pt]{article}

\usepackage{fullpage,times,color}
\usepackage{times}
\usepackage{graphicx}

\usepackage{amsthm}
\usepackage{amsmath}
\usepackage{amsfonts}
\usepackage{units}
\usepackage{caption}
\usepackage{subcaption}

\usepackage{natbib}

\usepackage{algorithm}
\usepackage{algorithmic}

\usepackage{hyperref}

\newcommand {\Par}[1]{ \left( #1 \right)}

\begin{document}

\author{
  Ohad Shamir\\
  Weizmann Institute of Science \\
  \texttt{ohad.shamir@weizmann.ac.il}
 \and
  Liran Szlak\\
  Weizmann Institute of Science \\
  \texttt{liran.szlak@weizmann.ac.il}
}

\date{}

\title{Online Learning with Local Permutations and Delayed Feedback} 
\maketitle

\vskip 0.3in

\newtheorem{theorem}{Theorem}
\newtheorem{lemma}{Lemma}
\newtheorem{corollary}{Corollary}

\begin{abstract}
We propose an Online Learning with Local Permutations (OLLP) setting, in which 
the learner is allowed to slightly permute the \emph{order} of the loss 
functions generated by an adversary. On one hand, this models natural 
situations where the exact order of the learner's responses is not crucial, and 
on the other hand, might allow better learning and regret performance, by 
mitigating highly adversarial loss sequences. Also, with random permutations, 
this can be seen as a setting interpolating between adversarial and stochastic 
losses. In this paper, we consider the 
applicability of this setting to convex online learning with delayed feedback, 
in which the feedback on the prediction made in round $t$ arrives with some 
delay $\tau$. With such delayed feedback, the best possible 
regret bound is well-known to be $O(\sqrt{\tau T})$. We prove that by 
being able to permute losses by a distance of at most $M$ (for $M\geq \tau$), 
the regret can be improved to $O(\sqrt{T}(1+\sqrt{\tau^2/M}))$, using a 
Mirror-Descent based algorithm which can be applied for both Euclidean and 
non-Euclidean geometries. We also prove a lower bound, showing 
that for $M<\tau/3$, it is impossible to improve the standard 
$O(\sqrt{\tau T})$ regret bound by more than constant factors. Finally, we 
provide some experiments validating the performance of our algorithm. 
\end{abstract}

\section{Introduction}
\label{sec:intro}

Online learning is traditionally posed as a repeated game where the learner 
has to provide predictions on an arbitrary sequence of loss functions, possibly 
even generated adversarially. Although it is often possible to devise 
algorithms with non-trivial regret guarantees, these have to cope with 
arbitrary loss sequences, which makes them conservative and in some cases 
inferior to algorithms not tailored to cope with worst-case behavior. Indeed, 
an emerging line of work considers how better online learning can be obtained 
on ``easy'' data, which satisfies some additional assumptions. Some examples 
include losses which are sampled i.i.d. from some distribution, change 
slowly in time, have a consistently best-performing predictor across time,  
have some predictable structure, mix adversarial and stochastic losses, etc. 
(e.g.
\cite{sani2014exploiting,NIPS2016_6341,bubeck2012best,seldin2014one,
hazan2010extracting,chiang2012online,steinhardt2014adaptivity,hazan2011better,
rakhlin2013online,seldin2014one}).

In this paper, we take a related but different direction: Rather than 
explicitly excluding highly adversarial loss sequences, we consider how 
slightly perturbing them can mitigate their worst-case behavior, and lead to 
improved performance. Conceptually, this resembles smoothed 
analysis \cite{spielman2004smoothed}, in which one 
considers the worst-case performance of some algorithm, after performing some 
perturbation to their input. The idea is that if the worst-case instances are 
isolated and brittle, then a perturbation will lead to easier instances, and 
better reflect the attainable performance in practice. 

Specifically, we propose a setting, in which the learner is allowed to slightly 
\emph{reorder} the sequence of losses generated by an adversary: Assuming the 
adversary chooses losses $h_1,\ldots,h_T$, and before any losses are revealed, 
the learner may choose a permutation 
$\sigma$ on $\{1,\ldots,T\}$, satisfying $\max_t|t-\sigma(t)|\leq M$ for some 
parameter $M$, and then play a standard 
online learning game on losses $h_{\sigma(1)},\ldots,h_{\sigma(T)}$. We denote 
this as the \emph{Online Learning with Local Permutations} (OLLP) setting. 
Here, $M$ controls the amount of power given to the learner: $M=0$ means that 
no reordering is performed, and the setting is equivalent to standard 
adversarial online learning. At the other extreme, $M=T$ means that the learner 
can reorder the 
losses arbitrarily. For example, the learner may choose to order the losses 
uniformly at random, making it a quasi-stochastic setting (the 
only difference compared to i.i.d. losses is that they are sampled 
without-replacement rather than with-replacement).  

We argue that allowing the learner some flexibility in the order of responses 
is a natural assumption. For example, when the learner needs to provide 
rapid predictions on a high-frequency stream of examples, it is often 
immaterial if the predictions are not provided in the exact same order at which 
the examples arrived. Indeed, by buffering examples for a few rounds before 
being answered, one can simulate the local permutations discussed earlier. 

We believe that this setting can be useful in various online learning problems, 
where it is natural to change a bit the order of the loss functions. In this 
paper, we focus on one well-known problem, namely online learning with delayed 
feedback. In this case, rather than being provided with the loss function 
immediately after prediction is made, the learner only receives the loss 
function after a certain number $\tau\geq 1$ of rounds. This naturally models 
situations where the feedback comes much more slowly than the required 
frequency of predictions: To give a concrete example, consider a web 
advertisement problem, where an algorithm picks an ad to display, and then 
receives a feedback from the user in the form of a click. It is likely 
that the algorithm will be required to choose ads for new users while still 
waiting for the feedback from the previous user. 

For convex online learning with delayed feedback, in a standard adversarial 
setting, it is known that the attainable regret is on the order of 
$O(\sqrt{\tau T})$, and this is also the best possible in the worst case
\cite{weinberger2002delayed,mesterharm2005line,langford2009slow,joulani2013online,quanrud2015online}.
On the other hand, in a stochastic setting where the losses are sampled i.i.d. 
from some distribution, \cite{agarwal2011distributed} show that the attainable 
regret is much better, on the order of $O(\sqrt{T}+\tau)$. This gap between the 
worst-case adversarial setting, and the milder i.i.d. 
setting, hints that this problem is a good fit for our OLLP framework.

Thus, in this paper, we focus on online learning with feedback delayed up to 
$\tau$ rounds, in the OLLP framework where the learner is allowed to locally 
permute the loss functions (up to a distance of $M$). First, we devise an 
algorithm, denoted as Delayed Permuted Mirror Descent, and prove that it 
achieves an expected regret bound of order $O (\sqrt{T(\tau^2/M+1)})$ assuming 
$M\geq \tau$. As $M$ increases 
compared to $\tau$, this regret bound interpolates between the standard 
adversarial $\sqrt{\tau T}$ regret, and a milder $\sqrt{T}$ regret, typical of 
i.i.d. losses. As its name implies, the algorithm is based on the well-known 
online mirror descent (OMD) algorithm (see 
\cite{hazan2016introduction,shalev2012online}), and works in the same 
generality, involving both Euclidean and non-Euclidean geometries. The 
algorithm is based on dividing the entire sequence of functions into blocks of 
size $M$ and performing a random permutation within each block. Then, two 
copies of OMD are ran on different parts of each block, with appropriate 
parameter settings. A careful analysis, mixing adversarial and stochastic 
elements, leads to the regret bound. 

In addition, we provide a lower bound complementing our upper bound analysis, 
showing that when $M$ is significantly smaller than $\tau$ (specifically, 
$\tau/3$), then even with local permutations, it is impossible to obtain a  
worse-case regret better than $\Omega(\sqrt{\tau T})$, matching (up to 
constants) the attainable regret in the standard adversarial setting where no 
permutations are allowed. Finally, we provide some experiments validating the 
performance of our algorithm. 

The rest of the paper is organized as follows: in 
section~\ref{sec:permutationSetting} we formally define the Online Learning 
with Local Permutation setting, section~\ref{sec:alg} describes the Delayed 
Permuted Mirror Descent algorithm and outlines its regret analysis, section 
~\ref{sec:lowerBound} discusses a lower bound for the delayed setting with 
limited permutation power, section~\ref{sec:exp} shows experiments, and finally 
section~\ref{sec:disc} provides concluding remarks, discussion, and open 
questions. Appendix \ref{app:proofs} contains most of the proofs.

\section{Setting and Notation}
\label{sec:permutationSetting}

\textbf{Convex Online Learning.} Convex online learning is posed as a repeated 
game between a learner and an adversary (assumed to be oblivious in this 
paper). First, the adversary chooses $T$ convex losses $h_1,\ldots,h_T$ which 
are functions from a convex set $\mathcal{W}$ to $\mathbb{R}$. At each 
iteration $t \in \{1,2,\ldots,T \}$, the learner makes a prediction $w_t$, and 
suffers a loss of $h_t \Par{w_t}$. To simplify the presentation, we use the 
same notation $\nabla h_t(w)$ to denote either a gradient of $h_t$ at $w$ (if 
the loss is differentiable) or a subgradient at $w$ otherwise, and refer to it 
in both cases as a gradient. We assume that both $w \in \mathcal{W}$ and the 
gradients of any function $h_t$ in any point $w\in 
\mathcal{W}$ are bounded w.r.t. some norm: Given a norm $\|\cdot\|$ with a dual 
norm $\|\cdot\|_{\ast}$, we assume that the diameter of the space $\mathcal{W}$ is bounded by $B^2$
and that $\forall w\in \mathcal{W}, \forall h \in \lbrace h_1,h_2,...,h_T 
\rbrace: \| \nabla h \left( w \right)\|_\ast \leq G$. The purpose of the 
learner is to minimize her (expected) regret,  i.e. 
\begin{gather*}
R(T) = \mathbb{E} \left[ \sum_{t=1}^T h_t \left( w_t \right) - \sum_{t=1}^T h_t \left( w^{\ast} \right) \right] 
~~~~\text{where}~~~~ w^{\ast} = \underset{w \in \mathcal{W}}{\text{argmin }}  \sum_{t=1}^T h_t \left( w \right)
\end{gather*}
where the expectation is with respect to the possible randomness of the algorithm. 

\textbf{Learning with Local Permutations.} In this paper, we introduce and 
study a variant of this standard setting, which gives the learner a bit more 
power, by allowing her to slightly modify the \emph{order} in which the losses 
are processed, thus potentially avoiding highly adversarial but brittle loss 
constructions. We denote this setting as the Online Learning with Local 
Permutations (OLLP) setting. Formally, letting $M$ be a permutation window 
parameter, the learner is allowed (at the beginning of the game, and before any 
losses are revealed) to permute $h_1,\ldots,h_T$ to 
$h_{\sigma^{-1}(1)},\ldots,h_{\sigma^{-1}(T)}$, where $\sigma$ is a permutation 
from the set $Perm := \left\lbrace \sigma : \forall t, | \sigma \Par{t} - t | 
\leq M \right\rbrace $. After this permutation is performed, the learner is 
presented with the permuted sequence as in the standard online learning 
setting, with the same regret as before.  To simplify notation, we let 
$f_t=h_{\sigma^{-1}(t)}$, so the learner is presented with the loss sequence 
$f_1,\ldots,f_T$, and the regret is the same as the standard regret, i.e. 
$R(T)=\mathbb{E} \left[ \sum_{t=1}^{T}f_t(w_t)-\sum_{t=1}^{T} f_t(w^*) \right]$.
Note that if $M=0$ then we are in the fully adversarial setting (no permutation is allowed). At the other extreme, if $M=T$ and $\sigma$ is chosen uniformly at random, then we are in a stochastic setting, with a uniform distribution over the set of functions chosen by the adversary (note that this is close but differs a bit from a setting of i.i.d. losses). In between, as $M$ varies, we get an interpolation between these two settings.

\textbf{Learning with Delayed Feedback.}
The OLLP setting can be useful in many applications, and can potentially lead to improved regret bounds for various tasks, compared to the standard adversarial online learning.
In this paper, we focus on studying its applicability to the task of learning from delayed feedback.

Whereas in standard online learning, the learner gets to observe the loss $f_t$ immediately at the end of iteration $t$, here we assume that at round $t$, she only gets to observe $f_{t-\tau}$ for some delay parameter $\tau < T$ (and if $t < \tau$, no feedback is received). For simplicity, we focus on the case where $\tau$ is fixed, independent of $t$, although our results can be easily generalized (as discussed in subsection~\ref{sec:extensionTau}). We emphasize that this is distinct from another delayed feedback scenario sometimes studied in the literature (\cite{agarwal2011distributed,langford2009slow}), where rather than receiving $f_{t-\tau}$ the learner only receives a (sub)gradient  of $f_{t-\tau}$ at $w_{t-\tau}$. This is a more difficult setting, which is relevant for instance when the delay is due to the time it takes to compute the gradient.

\section{Algorithm and Analysis}
\label{sec:alg}
Our algorithmic approach builds on the well-established online mirror descent 
framework. Thus, we begin with a short reminder of the Online Mirror Descent 
algorithm (see e.g. \cite{hazan2016introduction} for more details). Readers who are familiar with the algorithm are 
invited to skip to Subsection \ref{subsec:md}.

The online mirror descent algorithm is a generalization of online gradient 
descent, which can handle non-Euclidean geometries. The general idea is the 
following: we start with some point $w_t \in \mathcal{W}$, where $\mathcal{W}$ 
is our primal space. We then map this point to the dual space using a (striclty 
convex and continuously differentiable) mirror 
map $\psi$, i.e. $\nabla \psi \left( w_t \right) \in \mathcal{W}^{\ast}$, then 
perform the gradient update in the dual space, and finally map the resulting 
new point back to our primal space $\mathcal{W}$ again, i.e. we want to find a 
point $w_{t+1} \in \mathcal{W}$ s.t. $\nabla \psi \left( w_{t+1} \right) = 
\nabla \psi \left( w_t \right) - \eta \cdot g_t$ where $g_t$ denotes the 
gradient. Denoting by $w_{t+\frac{1}{2}}$ the point satisfying $\nabla \psi 
( w_{t+\frac{1}{2}} ) = \nabla \psi \left( w_t \right) - \eta \cdot 
g_t$, it can be shown that $w_{t+\frac{1}{2}} = \left( \nabla \psi^\ast 
\right)\left( \nabla \psi \left( w_t \right) - \eta \cdot g_t \right)$, where 
$\psi^\ast$ is the dual function of $\psi$. This point, $w_{t+\frac{1}{2}}$, 
might lie outside our hypothesis class $\mathcal{W}$, and thus we might need to 
project it back to our space $\mathcal{W}$. We use the Bregman divergence 
associated to $\psi$ to do this:
\[
w_{t+1} = \underset{w \in 
\mathcal{W}}{argmin} \triangle_{\psi} ( w, w_{t+\frac{1}{2}} ),
\] 
where the Bregman divergence $\Delta_{\psi}$ is defined as 
\[
\triangle_\psi 
\left( x,y \right) = \psi \left( x \right) - \psi \left( y \right) - \langle 
\nabla \psi ( y ), x-y \rangle.
\]

Specific choices of the mirror map $\psi$ leads to specific instantiations of the algorithms for various geometries. Perhaps the simplest example is $\psi 
\left( x \right) = \frac{1}{2} \| x \|_2^2$, with associated Bregman 
divergence  $\triangle_\psi \left( x,y \right) = \frac{1}{2} \cdot \| x - y 
\|^2$. This leads us to the standard and well-known online gradient descent 
algorithm, where $w_{t+1}$ is the Euclidean projection on the set $\mathcal{W}$ 
of 
$w_t-\eta\cdot g_t$.

Another example is the negative entropy mirror map $\psi \left( x 
\right) = \sum_{i=1}^n x_i \cdot \log \left( x_i \right)$, which is 
$1$-strongly convex with respect to the $1$-norm on the simplex $\mathcal{W} = 
\left\lbrace x \in \mathbb{R}^n_+ : \sum_{i=1}^n x_i = 1 \right\rbrace$. In 
that case, the resulting algorithm is the well-known multiplicative updates 
algorithm, where 
$w_{t+1,i} = w_{t,i}\cdot \exp(-\eta g_{t,i})/\sum_{j=1}^{n} 
w_{t,i}\cdot \exp(-\eta g_{t,j})$.
Instead of the $1$-norm on the simplex, one 
can also consider arbitrary $p$-norms, and take $\psi(x)=\frac{1}{2}\cdot 
\|x\|_q^2$, where $q$ is the dual norm (satisfying $1/p+1/q=1$).

\subsection{The Delayed Permuted Mirror Descent Algorithm}
\label{subsec:md}

Before describing the algorithm, we note that we will focus here on 
the case where the permutation window parameter $M$ is larger than the delay 
parameter $\tau$. If $M<\tau$, then our regret bound is generally no better 
than the $O(\sqrt{\tau T})$ obtainable by a standard algorithm without any 
permutations, and for $M<\tau/3$, this is actually tight as shown in Section 
\ref{sec:lowerBound}.

We now turn to present our algorithm, denoted as The Delayed Permuted 
Mirror Descent algorithm (see algorithm~\ref{alg:DPMD} below as well as 
figure~\ref{Fig:GradientUpdates} for a graphical illustration). First, the 
algorithm splits the time horizon $T$ into $M$ consecutive blocks, and performs 
a uniformly random permutation on the loss functions within each block. 
Then, it runs two online mirror descent algorithms in parallel, and uses 
the delayed gradients in order to update two separate predictors -- $w^f$ and 
$w^s$, where $w^f$ is used for prediction in the first $\tau$ rounds of each 
block, and $w^s$ is used for prediction in the remaining $M-\tau$ rounds (here, 
$f$ stands for ``first'' and $s$ stands for ``second''). The 
algorithm maintaining $w^s$ crucially relies on the fact that the gradient of 
any two functions in a block (at some point $w$) is equal, in expectation over 
the random permutation within each block. This allows us to avoid most of the 
cost incurred by delays within each block, since the expected gradient 
of a delayed function and the current function are equal. A complicating factor 
is that at the first $\tau$ rounds of each block, no losses from the current 
block has been revealed so far. To tackle this, we use another algorithm 
(maintaining $w^f$), specifically to deal with the losses at the beginning of 
each block. This algorithm does not benefit from the random permutation, and 
its regret scales the same as standard adversarial online learning with delayed 
feedback. However, as the block size $M$ increases, the \emph{proportion} of 
losses handled by $w^f$ decreases, and hence its influence on the 
overall regret diminishes.

The above refers to how the blocks are divided for purposes of 
\emph{prediction}. For purposes of \emph{updating} the predictor of each 
algorithm, we need to use the blocks a bit differently. Specifically, we let 
$T_1$ and $T_2$ be two sets of indices. $T_1$ includes all indices from the 
first $\tau$ time points of every block, and is used to update $w^f$. 
$T_2$ 
includes the first $M-\tau$ indices of every block, and is used to 
update $w^s$ (see figure~\ref{Fig:GradientUpdates}). Perhaps surprisingly, note 
that $T_1$ and $T_2$ are not disjoint, and their union does not cover all of 
$\{1,\ldots,T\}$. The reason is that due to the random permutation in each 
block, the second algorithm only needs to update on some of the loss functions 
in each block, in order to obtain an expected regret bound on all the losses it 
predicts on. 

\begin{figure*}[t]
\centering
\includegraphics[scale=0.7]{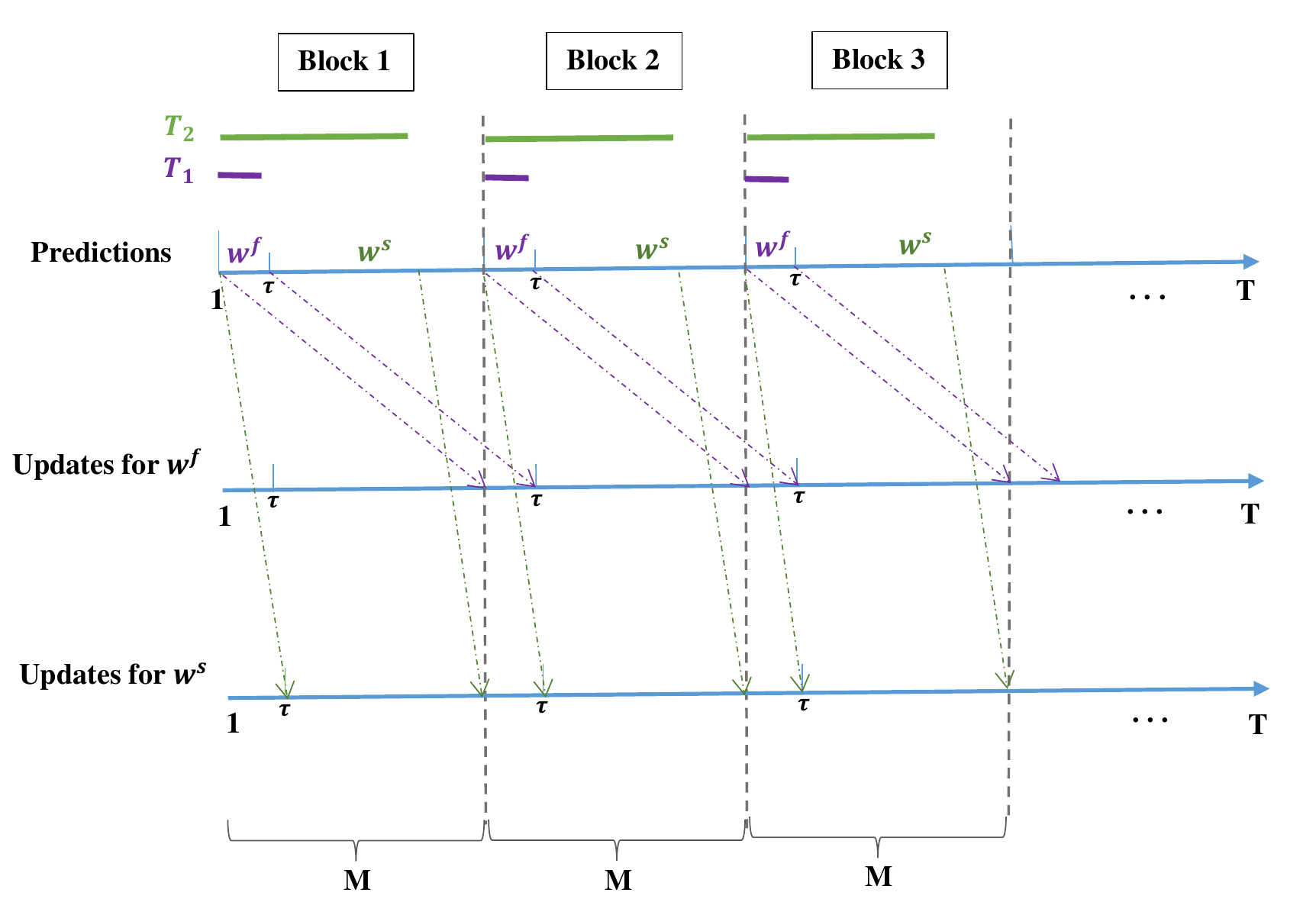}
\caption{Scheme of predictions and updates of both parallel algorithms (best 
viewed in color; see text for details). Top color bars mark which iterations 
are in $T_1$ (purple lines) and which are in $T_2$ (green lines). Top timeline 
shows which predictor, $w^f$ or $w^s$, is used to predict in each iteration. 
Middle timeline shows where gradients for updating $w^f$ come from (first 
$\tau$ iterations of the previous block), and lower timeline shows where 
gradients for updating $w^s$ come from (first $M-\tau$ iterations of the same 
block, each gradient from exactly $\tau$ rounds back).}
\label{Fig:GradientUpdates}
\end{figure*}

\begin{algorithm}[h]
\caption{Delayed Permuted Mirror Descent}
\label{alg:DPMD}
\begin{algorithmic}

\STATE {\bfseries Input:}{ $M$, $\eta_f$, $\eta_s$}
\STATE Init: $w_1^f = 0$, $w_1^s = 0$, $j_f = j_s = 1$ \\
\STATE Divide $T$ to consecutive blocks of size $M$, and permute the losses uniformly at random within each block. Let $f_1,\ldots,f_T$ denote the resulting permuted losses. \\
\FOR{$t = 1...,T$}{
	\IF{ $t \in $ first $\tau$ rounds of the block} {
	\STATE Predict using $w_{j_f}^f$ \\
	\STATE Receive a loss function from $\tau$ places back: $f_{t - M} = f_{T_1 \left( j_f - \tau \right)}$. If none exists (in the first $\tau$ iterations), take the $0$ function. \\
	\STATE Compute: $\nabla f_{T_1\left( j_f -\tau \right)} \left(w_{j_f - \tau}^f \right)$ \\
	\STATE Update: $w_{j_f+\frac{1}{2}}^f = \left( \nabla \psi^\ast \right)\left( \nabla \psi \left( w_ {j_f}^f \right) - \eta_f \nabla f_{T_1\left( j_f -\tau \right)} \left(w_{j_f - \tau}^f \right) \right)$ \\
\STATE Project: $w_{j_f+1} = \underset{w \in \mathcal{W}}{argmin} \triangle_{\psi} \left( w, w_{j_f+\frac{1}{2}}^f \right)$ \\
 	$j_f = j_f + 1$ \\
	}
	\ELSE {
	\STATE Predict using $w_{j_s}^s$ \\
	\STATE Receive a loss function from $\tau$ places back: $f_{t - \tau} = f_{T_2 \left( j_s\right)}$ \\ 	
	\STATE Compute: $\nabla f_{t-\tau} \left(w_{j_s}^s \right) = \nabla f_{T_2\left( j_s \right) } \left(w_{j_s}^s \right)$ \\
	\STATE Update: $w_{j_s+\frac{1}{2}}^s =  \left( \nabla \psi^\ast \right)\left( \nabla \psi \left( w_ {j_s}^s \right) - \eta_s \cdot \nabla f_{T_2\left( j_s \right) } \left(w_{j_s}^s \right) \right)$ \\	
\STATE Project: $w_{j_s+1} = \underset{w \in \mathcal{W}}{argmin} \triangle_{\psi} \left( w, w_{j_s+\frac{1}{2}}^s \right)$ \\
 	$j_s = j_s + 1$ \\	
	}
	\ENDIF
}
\ENDFOR
\end{algorithmic}
\end{algorithm}

\subsection{Analysis}
\label{subsec:analysis}
The regret analysis of the Delayed Permuted Mirror Descent algorithm is based 
on a separate analysis of each of the two mirror descent sub-algorithms, where 
in the first sub-algorithm the delay parameter $\tau$ enters multiplicatively, 
but doesn't play a significant role in the regret of the second sub-algorithm 
(which utilizes the stochastic nature of the permutations). Combining the 
regret bound of the two sub-algorithms, and using the fact that the portion of 
losses predicted by the second algorithm increases with $M$, leads to an 
overall regret bound improving in $M$.

In the proof, to analyze the effect of delay, we 
need a bound on the distance 
between any two consequent predictors $w_t,w_{t+1}$ generated by the 
sub-algorithm. This depends on the mirror map and Bregman divergence used for 
the update, and we currently do not have a bound holding in full generality. 
Instead, we let $\Psi_{(\eta_f,G)}$ be some upper bound on $\| 
w_{t+1}-w_t \|$, where the update is using step-size $\eta_f$ and gradients of 
norm $\leq G$. Using $\Psi_{\Par{\eta_f,G}}$ we prove a general bound for all 
mirror maps. In Lemmas \ref{lemma:euclidean} and \ref{lemma:negEntropy} in 
Appendix~\ref{app:fullProofAlg}, we show that for two common mirror maps 
(corresponding to online gradient descent and multiplicative weights), 
$\Psi_{(\eta_f,G)}\leq 
c\cdot \eta_f G$ for some numerical constant $c$, leading to a regret bound of 
$O(\sqrt{T(\tau^2/M+1)})$. 
Also, we prove theorem~\ref{thm:main} for $1$-strongly convex mirror maps, 
although it can be generalized to any $\lambda$-strongly convex mirror map 
by scaling.

\begin{theorem}\label{thm:main}
Given a norm $\|\cdot\|$, suppose that we run the Delayed Permuted Mirror 
Descent algorithm using a mirror map $\psi$ which is $1$-strongly convex w.r.t. 
$\|\cdot\|$, over a domain $\mathcal{W}$ with diameter $B^2$ w.r.t the bregman divergence of $\psi$: $\forall w,v\in \mathcal{W}: \triangle_\psi(w,v) \leq B^2$, and such 
that the (sub)-gradient $g$ of each loss function on any $w\in\mathcal{W}$ 
satisfies 
$\|g\|_*\leq G$ (where $\|\cdot\|_*$ is the dual norm of $\|\cdot\|$). Then the 
expected regret, given a delay parameter $\tau$ and step sizes $\eta_f,\eta_s$ 
satisfies:
\begin{align*}
\mathbb{E} \left[ \sum_{t=1}^{T} f_t \left( w_t \right) -  f_t \left( w^{\ast} \right) \right] & \leq \frac{B^2}{\eta_f} + \eta_f \cdot \frac{T \tau}{M} \cdot \frac{G^2}{2} 
 + \frac{T \tau^2}{M} \cdot G \cdot \Psi_{\Par{\eta_f,G}} + \frac{B^2}{\eta_s} 
 + \eta_s \cdot \frac{T \cdot \Par{M- \tau}}{M} \cdot \frac{G^2}{2}
\end{align*}
Furthermore, if $\Psi_{\Par{\eta_f,G}} \leq c \cdot \eta_f  G$ for some 
constant $c$, and $\eta_f = \frac{B \cdot \sqrt{M}}{G \cdot \sqrt{ T \cdot \tau 
\cdot \left( \frac{1}{2} + c \cdot \tau \right)}}$, $\eta_s = \frac{B \cdot 
\sqrt{2 M}}{G \cdot \sqrt{T \cdot \left( M - \tau \right) }}$, the regret is 
bounded by
\begin{align*}
& c \sqrt{\frac{T \tau}{M}} \cdot B G \sqrt{\frac{1}{2 } 
+ c \cdot \tau} + \sqrt{\frac{2 T \left( M - \tau \right)}{ M}} 
\cdot B G  
 ~=~ \mathcal{O} \Par{\sqrt{T} \cdot \Par{\sqrt{\frac{\tau^2}{M}} 
+ 1}}
\end{align*}
\end{theorem}

When $M=\mathcal{O}(\tau)$, this bound is $O(\sqrt{\tau T})$. 
similar to the standard adversarial learning case. However, as $M$ increases, 
the regret gradually improves to $O(\sqrt{T}+\tau)$, which is 
the regret attainable in a purely stochastic setting with i.i.d. losses.
The full proof can be found in appendix~\ref{app:ThmProofAlg}, and we sketch 
below the main ideas.

First, using the definition of regret, we show that it is enough to upper-bound 
the regret of each of the two sub-algorithms separately. Then, by a standard 
convexity argument, we reduce this to bounding sums of terms of the form
$\mathbb{E} [ \langle w^f_t - w_f^\ast, \nabla f_t (w^f_t) 
\rangle ]$ for the first sub-algorithm, and $\mathbb{E} \left[ 
\langle w^s_t - w_s^\ast, \nabla f_t \Par{w^s_t} \rangle \right]$ 
for the second sub-algorithm (where $w^\ast_f$ and $w_s^\ast$ are the best 
fixed points in hindsight for the losses predicted on by the first and second 
sub-algorithms, respectively, and where for simplicity we assume the losses are 
differentiable). In contrast, we can use the standard analysis of mirror 
descent, using delayed gradients, to get a bound for the somewhat different 
terms $\mathbb{E} [ \langle w^f_t - w_f^\ast, \nabla f_{t-\tau} 
(w^f_{t-\tau}) \rangle ]$ for the first sub-algorithm, and 
$\mathbb{E} \left[ 
\langle w^s_t - w_s^\ast, \nabla f_{t-\tau} \Par{w^s_t} \rangle 
\right]$ for the second sub-algorithm. Thus, it remains to bridge between these 
terms.

Starting with the second sub-algorithm, we note that since we performed a 
random permutation within each block, the expected value of all loss functions 
within a block (in expectation over the block, and evaluated at a fixed point) 
is equal. Moreover, at any time point, the predictor $w^s$ maintained by the 
second sub-algorithm does not depend on the delayed nor the current loss 
function. Therefore, conditioned on $w^s_t$, and in expectation over the random 
permutation in the block, we have that 
\begin{align*}
 \mathbb{E}[\nabla 
f_t(w_t^s)]=\mathbb{E}[f_{t-\tau}(w_t^s)]
\end{align*}
from which it can be shown that
\begin{align*}
\mathbb{E} \left[ 
\langle  
w^s_t - w_s^\ast, \nabla f_t (w^s_t)\rangle\right] = \mathbb{E} \left[ 
\langle  
w^s_t - w_s^\ast, \nabla f_{t-\tau} (w^s_t)\rangle\right]
\end{align*}
Thus, up to a 
negligible factor having to do with the first few rounds of the game, the 
second sub-algorithm's expected regret does not suffer from the delayed 
feedback.

For the first sub-algorithm, we perform an analysis which does not rely on the 
random permutation. Specifically, we first show that since we care just about 
the sum of the losses, it is sufficient to bound the difference between 
$\mathbb{E} 
[ \langle w^f_t - w_f^\ast, \nabla f_t (w^f_t) 
\rangle ]$ and $\mathbb{E} [ \langle w^f_{t+\tau} - w_f^\ast, \nabla 
f_{t} 
(w^f_{t}) \rangle ]$. Using Cauchy-Shwartz, this difference can 
be upper bounded by $\| w_t^f - w_{t+\tau}^f \| \cdot \| \nabla f_t (w^f_t)
\|$, which in turn is at most $c \cdot \tau \cdot \eta_f \cdot G^2$ using our 
assumptions on the gradients of the losses and the distance between consecutive 
predictors produced by the first sub-algorithm. 

Overall, we get two regret bounds, one for each sub-algorithm. The regret of 
the first sub-algorithm scales with $\tau$, similar to the no-permutation 
setting, but the sub-algorithm handles only a small fraction of the iterations 
(the first $\tau$ in every block of size $M$). In the rest of the iterations,  
where we use the second sub-algorithm, we get a bound that resembles more the 
stochastic case, without such dependence on $\tau$. Combining the two, the 
result stated in Theorem \ref{thm:main} follows.

\subsection{Handling Variable Delay Size}
\label{sec:extensionTau}
So far, we discussed a setting where the feedback arrives with a fixed delay of 
size $\tau$. However, in many situations the feedback might arrive with a 
variable delay size $\tau_t$ at any iteration $t$, which may raise a few 
issues. First, feedback might arrive in 
an asynchronous fashion, causing us to update our predictor using gradients 
from time points further in past after already using more recent gradients. 
This complicates the analysis of the algorithm. A second, algorithmic problem, 
is that we could also possibly receive multiple feedbacks simultaneously, or no 
feedback at all, in certain iterations, since the delay is of variable size. 
One simple solution is to use buffering and reduce the problem to a constant 
delay setting. Specifically, we assume that all delays are bounded by some 
maximal delay size $\tau$. We would like to use one gradient to update our 
predictor at every iteration (this is mainly for ease of analysis, practically 
one could update the predictor with multiple loss functions in a single 
iteration). In order to achieve this, we can use a buffer to store loss 
functions that were received but have not been used to update the predictors 
yet. We define $Grad_f$ and $Grad_s$, two buffers that will contain gradients 
from time points in $T_1$ or $T_2$, correspondingly. Each buffer is of size 
$\tau$. If we denote by $\mathcal{F}_t$ the set of function that have arrived 
in time $t$, we can simply store loss functions that have arrived 
asynchronously in the buffers defined above, sort them in ascending order, and 
take the delayed loss function from exactly $\tau$ iterations back in the 
update step. This loss function must be in the appropriate buffer since the 
maximal delay size is $\tau$. From this moment on, the algorithm can proceed as 
usual and its analysis still applies. 

\section{Lower Bound}
\label{sec:lowerBound}
In this section, we give a lower bound in the setting where $M < 
\frac{\tau}{3}$ with all feedback having delay of exactly $\tau$. We will show that for this case, the regret bound cannot be improved by more than a constant factor over the bound of the adversarial online learning problem with a fixed delay of size $\tau$, namely $\Omega \left( \sqrt{\tau T} \right)$ for a sequence of length $T$. We hypothesize that 
this regret bound also cannot be significantly improved for any $M = O(\tau)$ (and not just $\tau/3$). However, proving this remains an open problem.

\begin{theorem}
For every (possible randomized) algorithm $A$ with a permutation window of size $M \leq \frac{\tau}{3}$, there exists a choice of linear, $1$-Lipschitz functions over $[-1,1] \subset \mathbb{R}$, such that the expected regret of $A$ after $T$ rounds (with respect to the algorithm's randomness), is
\begin{gather*}
\mathbb{E} \left[ \sum_{t=1}^{T} f_t \Par{w_t} - \sum_{t=1}^{T} f_t 
\Par{w^{\ast}} \right] = \Omega \Par{\sqrt{\tau T} \right)  
~~~~\text{where }~~~~ 
w^{\ast} = \underset{w \in \mathcal{W}}{\text{argmin }} \sum_{t=1}^T f_t \left( 
w } 
\end{gather*}
\label{lemma:lowerBound}
\end{theorem}
For completeness, we we also provide in appendix~\ref{app:lowerBoundAdvers} a proof that when $M=0$ (i.e. no permutations allowed), then the worst-case 
regret is no better than $\Omega (\sqrt{\tau T})$. This is of course a special case of Theorem~\ref{lemma:lowerBound}, but applies to the standard adversarial online setting (without any local permutations), and the proof is simpler. The proof sketch for the setting where no permutation is allowed was already provided in \cite{langford2009slow}, and our contribution is in providing a full formal proof. 

The proof in the case where $M=0$ is based on linear losses of the form $f_t = 
\alpha_t \cdot w_t$ over $[-1,+1]$, where $\alpha_t\in \{-1,+1\}$. Without 
permutations, it is possible to prove a $\Omega(\sqrt{\tau T})$ lower bound by 
dividing the $T$ iterations into blocks of size $\tau$, where the $\alpha$ 
values of all losses at each block is the same and randomly chosen to equal 
either $+1$ or $-1$. Since the learner does not obtain any information about 
this value until the block is over, this reduces to adversarial online learning 
over $T/\tau$ rounds, where the regret at each round scales linearly with 
$\tau$, and overall regret at least $\Omega(\tau\sqrt{T/\tau}) = 
\Omega(\sqrt{\tau T})$. 

In the proof of theorem~\ref{lemma:lowerBound}, we show that by using a similar 
construction, even with permutations, having a permutation window less than 
$\tau/3$ still means that the $\alpha$ values would still be unknown until all loss functions of the block are processed, leading to the same lower bound up 
to constants. 

The formal proof appears in the appendix, but can be sketched as follows: 
first, we divide the $T$ iterations into blocks of size $\tau/3$. Loss 
functions within each block are identical, of the form $f_t = \alpha_t \cdot 
w_t$, and the value of $\alpha$ per block is chosen uniformly at random from 
$\{-1,+1\}$, 
as before. Since here, the permutation window $M$ is smaller than $\tau/3$, 
then even after permutation, the time difference between the first and last 
time we encounter an $\alpha$ that originated from a single block is less than 
$\tau$. This means that by the time we get any information on the $\alpha$ in a 
given block, the algorithm already had to process all the losses in the block, 
which leads to the same difficulty as the no-permutation 
setting. Specifically, since the predictors chosen by the algorithm when 
handling the losses of the block do not depend on the $\alpha$ value in that 
block, and that $\alpha$ is chosen randomly, we get that the expected loss of 
the algorithm at any time point $t$ equals $0$. Thus, the cumulative loss 
across the entire loss sequence is also $0$. In
contrast, for $w^\ast$, the optimal predictor in hindsight over the entire 
sequence, we can prove an expected accumulated loss of $-\Omega (\sqrt{\tau 
T})$ after $T$ iterations, using Khintchine inequality and the fact that the 
$\alpha$'s were randomly chosen per block. 
This leads us to a lower bound of expected regret of order $\sqrt{\tau T}$, for any algorithm with a local permutation window of size $M < \tau/3$.

\section{Experiments}
\label{sec:exp}
We consider the adversarial setting described in section~\ref{sec:lowerBound}, 
where an adversary chooses a sequence of functions such that every $\tau$ 
functions are identical, creating blocks of size $\tau$ of identical loss 
functions, of the form $f_t(w_t) = \alpha_t \cdot w_t$ where 
$\alpha_t$ is chosen randomly in $\{-1,+1\}$ for each block. In all experiments 
we use $T=10^5$ rounds, a delay parameter of $\tau=200$, set our step sizes 
according to the theoretical analysis, and report the mean regret value over 
1000 repetitions of the experiments.

\begin{figure*}[th]
\centering
\includegraphics[scale=0.18]{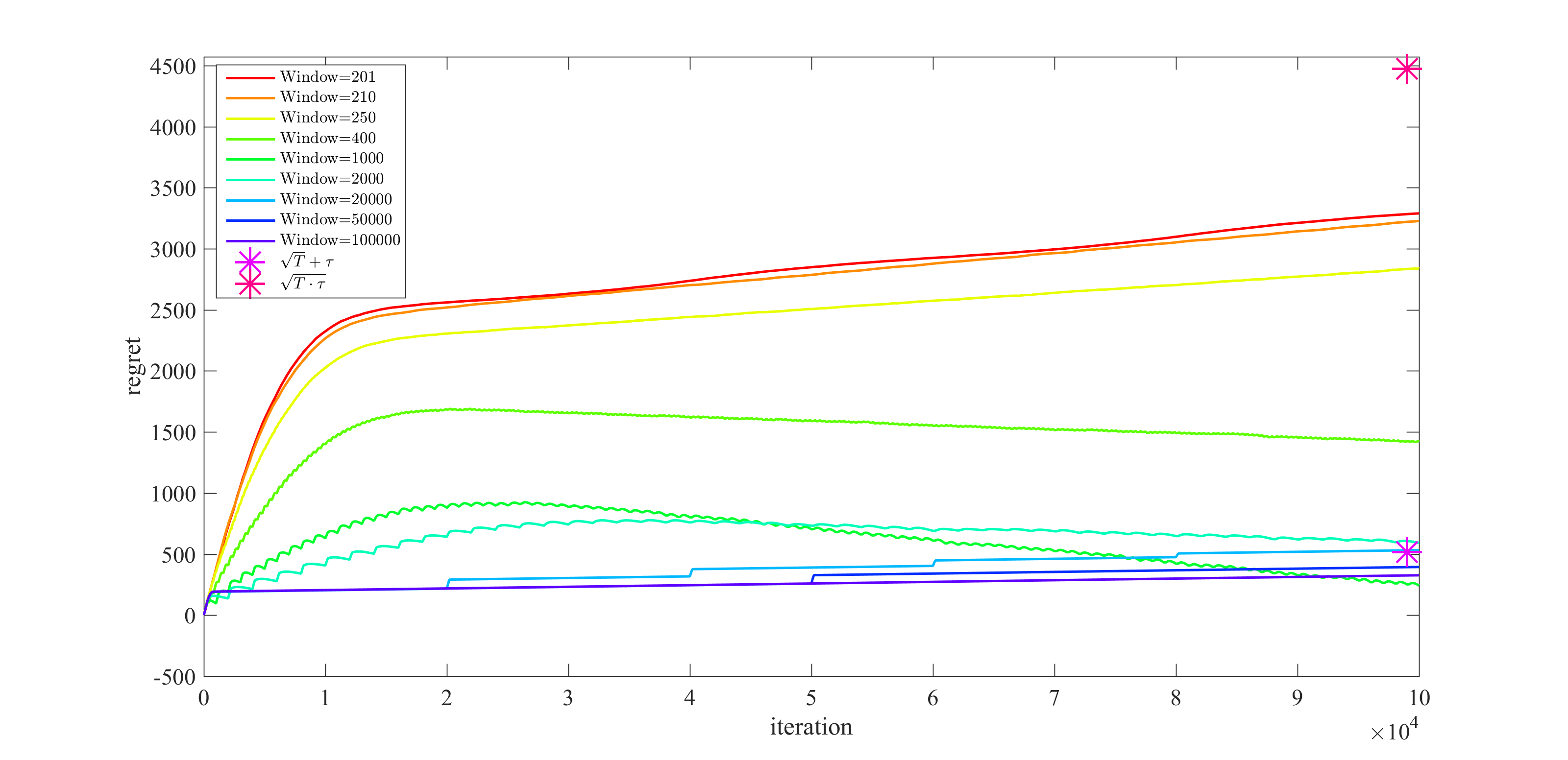}
\caption{Regret of the Delayed Permuted Mirror Descent algorithm, with local 
permutation in window sizes ranging from $M=\tau+1$ to $M=T$. A pink $\ast$ 
indicates the order of the stochastic bound ($\sqrt{T}+\tau$), and a red $\ast$ 
indicates the order of the adversarial bound ($\sqrt{\tau T}$). Regret is 
averaged over 1000 repetitions. Best viewed in color.}
\label{Fig:RegretVSiteration}
\end{figure*}

In our first experiment, we considered the behavior of our Delayed Permuted 
Mirror Descent algorithm, for window sizes $M>\tau$, ranging from $\tau+1$ to 
$T$. In this experiment, we chose the $\alpha$ values randomly, while ensuring 
a gap of 200 between the number of blocks with $+1$ values and the number of 
blocks with $-1$ values (this ensures that the optimal $w^\ast$ is a 
sufficiently strong competitor, since otherwise the setting is too ``easy'' and 
the algorithm can attain negative regret in some situations).  The results are 
shown in Figures~\ref{Fig:RegretVSiteration} and \ref{Fig:RegretVS_M}, where 
the first figure presents the accumulated regret of our algorithm over time, 
whereas the second figure presents the overall regret after $T$ rounds, as a 
function of the window size $M$.


When applying our 
algorithm in this setting with different values of $M > \tau$, ranging from 
$M=\tau + 1$ 
and up to $M=T$, we get a regret that scales from the order of the adversarial bound to the order of the stochastic bound depending on the window size, as expected by our analysis. For all window sizes greater than $5 \cdot \tau$, we get a regret that is in the order of the stochastic bound - this is not surprising, since after the 
permutation we get a sequence of functions that is very close to an i.i.d. 
sequence, in which case 
any algorithm can be shown to achieve $O(\sqrt{T})$ regret in expectation. Note 
that this performance is better than that predicted by our theoretical 
analysis, which implies an $O(\sqrt{T})$ behavior only when 
$M\geq\Omega(\tau^2)$. It is an open and interesting question whether it means 
that our analysis can be improved, or whether there is a harder construction 
leading to a tighter lower bound. 

\begin{figure}[h]
\centering
\includegraphics[scale=0.35]{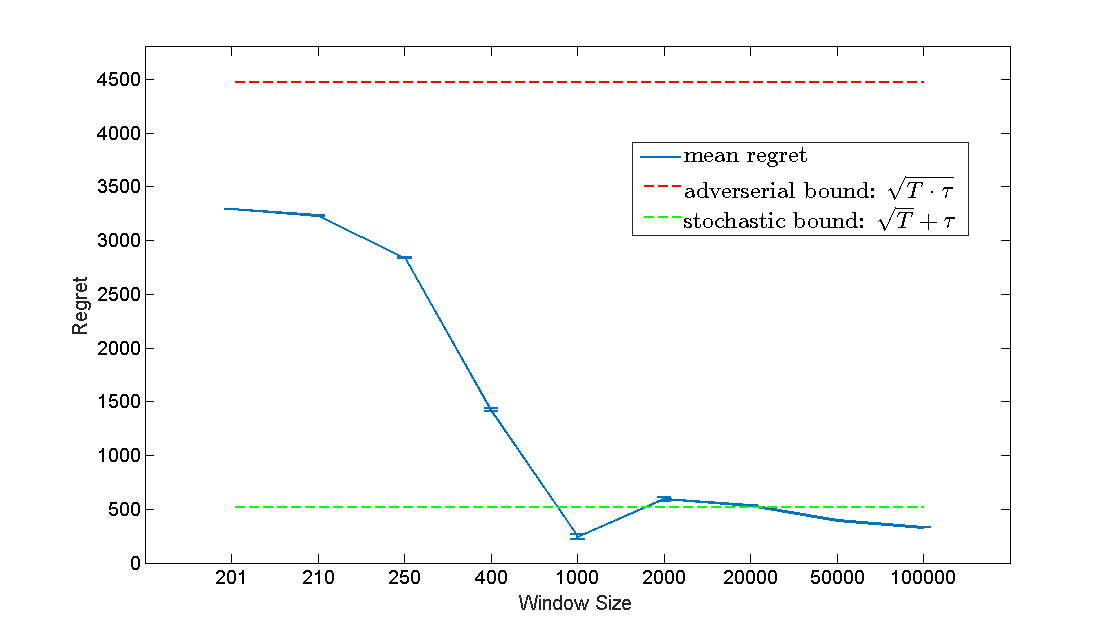}
\caption{Regret of the Delayed Permuted Mirror Descent algorithm for different 
window sizes, after $T=10^5$ iterations, with local permutation window sizes 
ranging from $M=\tau+1$ to $M=T$. Red dashed line (top) indicates the order of 
the adversarial bound ($\sqrt{\tau T}$) and green dashed line (bottom) 
indicates the order of the stochastic bound ($\sqrt{T}+\tau$). Regret is 
averaged over 1000 repetitions, error bars indicate standard error of the mean. 
Best viewed in color.}
\label{Fig:RegretVS_M}
\end{figure}

\begin{figure}[h]
\centering
\includegraphics[scale=0.35]{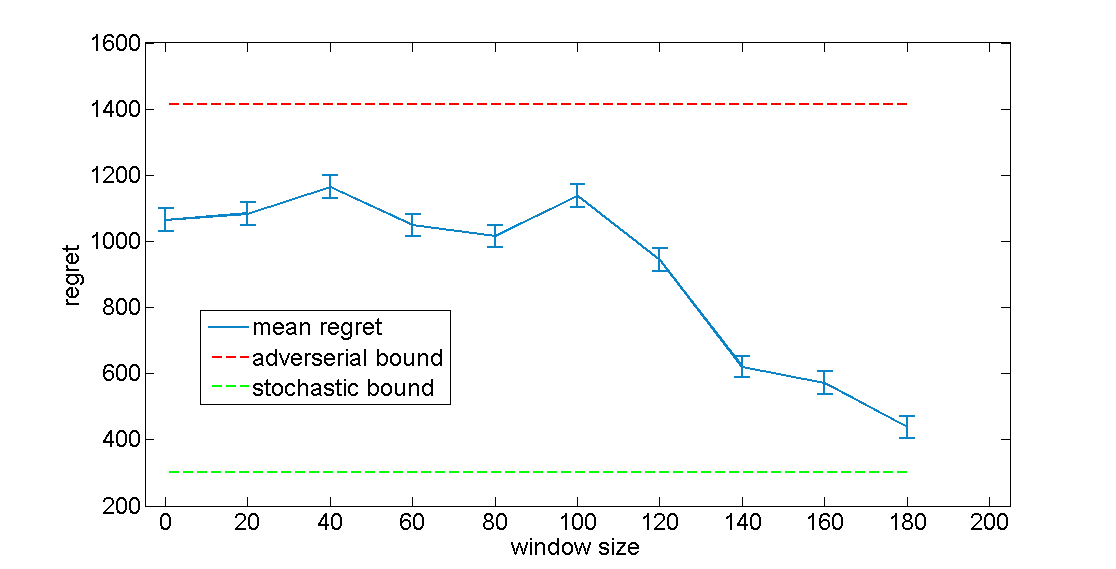}
\caption{Regret of the standard Online Gradient Descent algorithm, in a 
adversarialy designed setting as described in~\ref{sec:lowerBound}, and with 
local permutation in window sizes ranging from $M=0$ to $M=\frac{9}{10} \tau$. 
Red dashed line (top) indicates the order of the adversarial bound ($\sqrt{\tau 
T}$) and green dashed line (bottom) indicates the order of the stochastic bound 
($\sqrt{T}+\tau$). Regret is averaged over 1000 repetitions, error bars 
indicate standard error of the mean. Best viewed in color.}
\label{Fig:OGDfigure}
\end{figure}

In our second experiment, we demonstrate the brittleness of the lower bound 
construction for standard online learning with delayed feedback, focusing on 
the $M<\tau$ regime. Specifically,
we create loss functions with blocks as before (where following the lower bound 
construction, the $\alpha$ values in each block of size $\tau=200$ is chosen 
uniformly at random). Then, we perform a random permutation over consecutive 
windows of size $M$ (ranging from $M=0$ up to $M=\frac{9}{10}\tau$ in intervals 
of $\frac{1}{10}\tau$). Finally, we run standard Online Gradient Descent with 
delayed gradients (and fixed step size $1/\sqrt{T}$), on the permuted losses. 
The results are presented in Figure \ref{Fig:OGDfigure}.

For window sizes $M<\frac{\tau}{2}$ we see that the regret is close to the 
adversarial bound, whereas as we increase the window size the regret decreases 
towards the stochastic bound. This experiment evidently shows that this 
hardness construction is indeed brittle, and easily breaks in the face of local 
permutations, even for window sizes $M<\tau$. 

\section{Discussion}
\label{sec:disc}
We presented the OLLP setting, where a learner can locally permute the sequence of examples from which she learns. This setting can potentially allow for improved learning in many problems, where the worst-case regret is based on highly adversarial yet brittle constructions. 
In this paper, we focused on the problem of learning from delayed feedback in 
the OLLP setting, and showed how it is possible to improve the regret by 
allowing local permutations. Also, we proved a lower bound in the situation 
where the permutation window is significantly smaller than the feedback delay, 
and showed that in this case, permutations cannot allow for a better regret 
bound than the standard adversarial setting. We also provided some experiments, 
demonstrating the power of the setting as well as the feasibility of the 
proposed algorithm. An interesting open question is what minimal permutation 
size allows non-trivial regret improvement, and whether our upper bound in 
Theorem~\ref{thm:main} is tight. As suggested by our empirical 
experiments, it is possible that even small local permutations are enough to 
break highly adversarial sequences and improve performance in otherwise 
worst-case scenarios. 
Another interesting direction is to extend our results to a partial feedback 
(i.e. bandit) setting. Finally, it would be interesting to study other cases 
where local permutations 
allow us to interpolate between fully adversarial and more benign online 
learning scenarios.

\subsection*{Acknowledgements}

OS is supported in part by an FP7 Marie Curie CIG grant, the Intel ICRI-CI Institute, and Israel Science Foundation grant 425/13.

\clearpage

\bibliography{delayed_feedback_new}
\bibliographystyle{plainnat}

\appendix
\onecolumn
\section{Proofs}
\label{app:proofs}

\subsection{Analysis Of The Delayed Permuted Mirror Descent Algorithm}
\label{app:fullProofAlg}
We will use throughout the proofs the well known Pythagorean Theorem for Bregman divergences, and the 'projection' lemma that considers the projection step in the algorithm.

\begin{lemma}{Pythagorean Theorem for Bregman divergences} \\
Let $v$ be the projection of $w$ onto a convex set $\mathcal{W}$ w.r.t Bregman divergence $\triangle_\psi$: $v = argmin_{u \in \mathcal{W}} \triangle_\psi \left( u,w \right)$, then: $\triangle_\psi \left( u,w \right) \geq \triangle_\psi \left( u,v \right) + \triangle_\psi \left( v,w \right)$
\end{lemma}

\begin{lemma}{Projection Lemma} \\
Let $\mathcal{W}$ be a closed convex set and let $v$ be the projection of $w$ onto $\mathcal{W}$, namely, \\
$v = \underset{x \in \mathcal{W}}{argmin } \| x-w \|^2$. Then, for every $u \in \mathcal{W}$, $\| w - u \|^2 - \| v - u \|^2 \geq 0$
\label{lemma:projection}
\end{lemma}

The following lemma gives a bound on the distance between two consequent predictions when using the Euclidean mirror map:
\begin{lemma}
Let $g \in \mathbb{R}^n $ s.t. $\| g \|_2 < G$, $\mathcal{W}$ a convex set, and $\eta > 0$ be fixed. 
Let $w \in \mathcal{W}$ and $w_2 = w - \eta \cdot g$. Then, for $w' = \underset{u \in \mathcal{W}}{argmin} \| w_2 - u\|_2^2$, we have that $\| w - w' \| \leq \eta \cdot G$
\label{lemma:euclidean}
\end{lemma}
\begin{proof}
From the projection lemma: $\| w_2 - w \|_2^2 \geq \| w' - w \|_2^2$ and so: $\| w_2 - w \|_2 \geq \| w' - w \|_2$. From definition: $\| w_2 - w \|_2 = \| \eta \cdot g\|_2 \leq \eta \cdot G$. and so we get: $\| w' - w \|_2 \leq \| w_2 - w \|_2 \leq \eta \cdot G$
\end{proof}

We prove a modification of Lemma 2 given in \cite{menache2014demand} in order to bound the distance between two consequent predictions when using the negative entropy mirror map:
\begin{lemma}
Let $g \in \mathbb{R}^n $ s.t. $\| g \|_1 \leq G$ for some $G>0$ and let $\eta > 0$ be fixed, with $\eta < \frac{1}{\sqrt{2} \cdot G}$. For any distribution vector $w$ in the $n-simplex$, if we define $w'$ to be the new distribution vector 
\begin{align*}
\forall i \in \left\lbrace 1,...,n \right\rbrace, w'_i = \frac{w_i \cdot \exp \left(- \eta \cdot g_i \right)}{\sum_{j=1}^{n} w_j \cdot \exp \left( - \eta \cdot g_j \right)}
\end{align*}
Then $\| w - w' \|_1 \leq 3 \eta G$
\label{lemma:negEntropy}
\end{lemma}
\begin{proof}
Since $\| g \|_\infty < G$ and $\eta < \frac{1}{\sqrt{2} \cdot G}$ we get that $\forall i : | \eta \cdot g_i | < 1$. 
We have that:
\begin{align*}
\| w - w' \|_1 = \sum_{i=1}^n | w_i - w'_i| = \sum_{i=1}^n \left| w_i \cdot \left( 1- \frac{ \exp \left(- \eta \cdot g_i \right)}{\sum_{j=1}^{n} w_j \cdot \exp \left( - \eta \cdot g_j \right)} \right) \right| 
\end{align*}
Since $\| w \|_1 = 1$, we can apply Holder's inequality, and upper bound the above by
\begin{align*}
\underset{i}{max} \left| 1- \frac{ \exp \left(- \eta \cdot g_i \right)}{\sum_{j=1}^{n} w_j \cdot \exp \left( - \eta \cdot g_j \right)} \right|
\end{align*}
Using the inequality $1-x \leq \exp \left( - x \right) \leq \frac{1}{1+x}$ for all $| x | \leq 1$, we know that
\begin{align*}
1 - \eta \cdot g_i \leq \exp \left( - \eta \cdot g_i \right) \leq \frac{1}{1+\eta \cdot g_i}
\end{align*}
and since $- \eta G \leq \eta \cdot g_i \leq \eta G$ we have that
\begin{align*}
1 - \eta \cdot g_i \leq \exp \left( - \eta \cdot g_i \right) \leq \frac{1}{1+\eta \cdot g_i} \leq \frac{1}{1-\eta G}  
\end{align*}
and so we get: 
\begin{align*}
1 - \frac{\frac{1}{1+ \eta g_i}}{1+ \eta G} \leq 1- \frac{ \exp \left(- \eta \cdot g_i \right)}{\sum_{j=1}^{n} w_j \cdot \exp \left( - \eta \cdot g_j \right)} \leq 1 - \frac{1- \eta \cdot g_i}{\frac{1}{1-\eta G}}
\end{align*}
Using again the fact that $- \eta G \leq \eta \cdot g_i \leq \eta G$, we have
\begin{align*}
1 - \frac{\frac{1}{1 - \eta G}}{1+ \eta G} \leq 1- \frac{ \exp \left(- \eta \cdot g_i \right)}{\sum_{j=1}^{n} w_j \cdot \exp \left( - \eta \cdot g_j \right)} \leq 1 - \frac{1 - \eta G}{\frac{1}{1-\eta G}}
\end{align*}
\begin{align*}
\Longrightarrow \frac{- \eta^2 G^2}{1 - \eta^2 G^2} = 1 - \frac{1}{1 - \eta^2 G^2} \leq 1- \frac{ \exp \left(- \eta \cdot g_i \right)}{\sum_{j=1}^{n} w_j \cdot \exp \left( - \eta \cdot g_j \right)} \leq 1 - \left( 1- \eta G\right)^2 = 2 \eta G + \eta^2 G^2
\end{align*}
Now, since $\eta G < 1$, we get that:
\begin{align*}
\frac{- \eta^2 G^2}{1 - \eta^2 G^2} \leq 1- \frac{ \exp \left(- \eta \cdot g_i \right)}{\sum_{j=1}^{n} w_j \cdot \exp \left( - \eta \cdot g_j \right)} \leq 2 \eta G + \eta G = 3 \eta G
\end{align*}
and so we can conclude that 
\begin{align*}
\underset{i}{max} \left| 1- \frac{ \exp \left(- \eta \cdot g_i \right)}{\sum_{j=1}^{n} w_j \cdot \exp \left( - \eta \cdot g_j \right)} \right| \leq \underset{i}{max} \left( \left| \frac{- \eta^2 G^2}{1 - \eta^2 G^2} \right|, \left| 3 \eta G \right| \right) \leq \underset{i}{max} \left(  \frac{ \eta G}{1 - \eta^2 G^2} , 3 \eta G \right)
\end{align*}
Since $\eta <\frac{1}{\sqrt{2} G}$, we get $\Par{\eta \cdot G}^2 < \frac{1}{2}$. Thus we get:
\begin{align*}
\underset{i}{max} \left| 1- \frac{ \exp \left(- \eta \cdot g_i \right)}{\sum_{j=1}^{n} w_j \cdot \exp \left( - \eta \cdot g_j \right)} \right| \leq \underset{i}{max} \left(  2 \eta G , 3 \eta G \right) \leq 3 \eta G
\end{align*}
which gives us our desired bound.
\end{proof}

With the above two lemmas in hand, we bound the distance between consequent predictors by $ c \eta G$, where $c$ is a different constant in each mirror map: $c=1$ for the euclidean case, and $c=3$ for the negative entropy mirror map.\\
Note that both mapping are $1$-strongly convex with respect to their respective norms. For other mappings with a different strong convexity constant, one would need to scale the step sizes according to the strong convexity parameter in order to get the bound.

\subsubsection{Proof of Theorem \ref{thm:main}}
\label{app:ThmProofAlg}
We provide an upper bound on the regret of the algorithm, by competing against the best fixed action in each one of the sets of iterations- the first $\tau$ iterations and the last $M-\tau$ iterations in each block. This is an upper bound on competing against the best fixed predictor in hindsight for the entire sequence.
Formally, we bound:
\begin{align*}
R(T) & = \mathbb{E} \left[ \sum_{t=1}^T f_t \left( w_t \right) - \sum_{t=1}^T f_t \left( w^{\ast} \right) \right] \\ 
& \leq \mathbb{E} \left[ \sum_{i=0}^{\frac{T}{M} -1} \left( \sum_{t=M\cdot i +1}^{M \cdot i + \tau} f_t \left( w_t \right) - f_t \left( w_f^{\ast} \right) + \sum_{t=M \cdot i + \tau +1}^{M \cdot \left( i+1 \right) } f_t \left( w_t \right) - f_t \left( w_s^{\ast} \right) \right) \right] \\
& \text{where } \\
& w_f^{\ast} = \underset{w \in \mathcal{W}}{\text{argmin }}  \sum_{i=0}^{\frac{T}{M} -1} \sum_{t=M\cdot i +1}^{M \cdot i + \tau} f_t \left( w \right) \text{and } w_s^{\ast} = \underset{w \in \mathcal{W}}{\text{argmin }} \sum_{i=0}^{\frac{T}{M} -1} \sum_{t=M\cdot i + \tau +1}^{M \cdot \left(i+1\right) } f_t \left( w \right) &
\end{align*}
where expectation is taken over the randomness of the algorithm.

The diameter of the domain $\mathcal{W}$ is bounded by $B^2$, and so $\triangle_\psi \left( w_f^\ast, w^f_0 \right) \leq B^2$ and $\triangle_\psi \left( w_s^\ast, w^s_0 \right)\leq B^2$. 
We start with a general derivation that will apply both for $w^s$ and for $w^f$ simultaneously. For the following derivation we use the notation $w_j, w_{j+1}$ omitting the $f,s$ superscript, for denoting subsequent updates of the predictor vector, whether it is $w^s$ or $w^f$. 

Denote by $g_j$ the gradient used to update $w_j$, i.e., $\nabla \psi \left( w_{j + \frac{1}{2}} \right) = \nabla \psi \left( w_j \right) - \eta \cdot g_j$, and $w_{j+1} = \underset{w \in \mathcal{W}}{argmin} \triangle_{\psi} \left( w, w_{j+\frac{1}{2}} \right)$. \\
Looking at the update step in the algorithm, we have that $g_j = \frac{1}{\eta} \cdot \left( \nabla \psi \left( w_j \right) - \nabla \psi \left( w_{j + \frac{1}{2}} \right) \right)$ and thus:
\begin{align*}
\left\langle w_j - w^{\ast}, g_j \right\rangle & = \frac{1}{\eta} \cdot \left\langle w_j - w^{\ast}, \left( \nabla \psi \left( w_j \right) - \nabla \psi \left( w_{j + \frac{1}{2}} \right) \right) \right\rangle \\
& = \frac{1}{\eta} \cdot \left( \triangle_\psi \left( w^{\ast}, w_j \right) + \triangle_\psi \left( w_j, w_{j + \frac{1}{2}} \right) - \triangle_\psi \left( w^{\ast}, w_{j + \frac{1}{2}} \right) \right)
\end{align*}
We now use the Pythagorean Theorem to get:
\begin{align*}
\leq \frac{1}{\eta} \cdot \left( \triangle_\psi \left( w^{\ast}, w_j \right) + \triangle_\psi \left( w_j, w_{j + \frac{1}{2}} \right) - \triangle_\psi \left( w^{\ast}, w_{j+1} \right)  - \triangle_\psi \left( w_{j+1}, w_{j+\frac{1}{2}} \right) \right)
\end{align*}
When we sum terms for all updates of the predictor, $w^f$ or $w^s$ respectively, the terms $\triangle_\psi \left( w^{\ast}, w_j \right) - \triangle_\psi \left( w^{\ast}, w_{j+1} \right)$ will result in a telescopic sum, canceling all terms expect the first and last. Thus we now concentrate on bounding the term: $\triangle_\psi \left( w_j, w_{j + \frac{1}{2}} \right) - \triangle_\psi \left( w_{j+1}, w_{j+\frac{1}{2}} \right)$.
\begin{align*}
& \triangle_\psi \left( w_j, w_{j + \frac{1}{2}} \right) - \triangle_\psi \left( w_{j+1}, w_{j+\frac{1}{2}} \right) = \psi \left( w_j \right)  - \psi \left( w_{j+1} \right) - \left\langle w_j - w_{j+1}, \nabla \psi \left( w_{j + \frac{1}{2}} \right) \right\rangle \\
& \underset{ \mbox{ $\psi$ 1-strong convex}}{\leq} \left\langle w_j - w_{j+1}, \nabla \psi \left(w_j \right) - \nabla \psi \left( w_{j+\frac{1}{2}} \right) \right\rangle - \frac{1}{2} \cdot \| w_j - w_{j+1} \|^2 \\
& = \left\langle w_j - w_{j+1}, \eta \cdot g_j \right\rangle - \frac{1}{2} \cdot \| w_j - w_{j+1} \|^2 \\
& \leq \eta \cdot G \cdot \| w_j - w_{j+1} \| - \frac{1}{2} \cdot \| w_j - w_{j+1} \|^2 \\
& \leq \frac{\left( \eta \cdot G \right)^2}{2} &
\end{align*}
where the last inequality stems from the fact that $\left( \| w_j -w_{j+1} \| \cdot \frac{\sqrt{1}}{\sqrt{2}} - \frac{\eta \cdot G}{\sqrt{2}} \right)^2 \geq 0$ 

We now continue with the analysis referring to $w^f$ and $w^s$ separately.
Summing over $j=\tau + 1$ to $\left( \frac{T}{M} +1 \right) \cdot \tau$ for $w^f$ (these are the $\frac{T}{M} \tau$ iterations in which the first sub-algorithm is in use), and from $j=1$ to $\frac{T}{M} \cdot \left( M - \tau \right)$ for $w^s$ (these are the $\frac{T}{M} (M - \tau)$ iterations in which the second sub-algorithm is in use) we get: 

For $w^f$: 
\begingroup
\allowdisplaybreaks
\begin{align*}
& \sum_{j=\tau + 1}^{\left( \frac{T}{M} +1 \right) \cdot \tau} \left\langle w^f_j - w_f^{\ast}, g_j \right\rangle \\
& = \sum_{j=\tau + 1}^{\left( \frac{T}{M} +1 \right) \cdot \tau} \left\langle w^f_j - w_f^{\ast}, \nabla f_{T_1 \left( j - \tau \right)} \left( w_{j-\tau}^f \right) \right\rangle \\
& = \sum_{j=\tau + 1}^{\left( \frac{T}{M} +1 \right) \cdot \tau} \frac{1}{\eta} \cdot \left\langle w^f_j - w_f^{\ast}, \left( \nabla \psi \left( w^f_j \right) - \nabla \psi \left( w^s_{j + \frac{1}{2}} \right) \right) \right\rangle \\
& = \sum_{j=\tau + 1}^{\left( \frac{T}{M} +1 \right) \cdot \tau} \frac{1}{\eta} \cdot \left( \triangle_\psi \left( w_f^{\ast}, w^f_j \right) + \triangle_\psi \left( w^f_j, w^f_{j + \frac{1}{2}} \right) - \triangle_\psi \left( w_f^{\ast}, w^f_{j + \frac{1}{2}} \right) \right) \\
& \leq \sum_{j=\tau + 1}^{\left( \frac{T}{M} +1 \right) \cdot \tau} \frac{1}{\eta} \cdot \left( \triangle_\psi \left( w_f^{\ast}, w^f_j \right) + \triangle_\psi \left( w^f_j, w^f_{j + \frac{1}{2}} \right) - \triangle_\psi \left( w_f^{\ast}, w^f_{j+1} \right)  - \triangle_\psi \left( w^f_{j+1}, w^f_{j+\frac{1}{2}} \right) \right) \\
& \leq \frac{1}{\eta} \cdot \sum_{j=\tau + 1}^{\left( \frac{T}{M} +1 \right) \cdot \tau} \triangle_\psi \left( w_f^{\ast}, w^f_j \right) - \triangle_\psi \left( w_f^{\ast}, w^f_{j+1} \right) + \frac{1}{\eta} \cdot \sum_{j=\tau + 1}^{\left( \frac{T}{M} +1 \right) \cdot \tau} \triangle_\psi \left( w^f_j, w^f_{j + \frac{1}{2}} \right) - \triangle_\psi \left( w^f_{j+1}, w^f_{j+\frac{1}{2}} \right) \\
& = \frac{1}{\eta} \cdot \triangle_\psi \left( w_f^{\ast}, w^f_{\tau +1} \right) - \triangle_\psi \left( w_f^{\ast}, w^f_{\left( \frac{T}{M} +1 \right) \cdot \tau} \right) + \frac{1}{\eta} \cdot \sum_{j=\tau + 1}^{\left( \frac{T}{M} +1 \right) \cdot \tau} \triangle_\psi \left( w_j^f, w^f_{j + \frac{1}{2}} \right) - \triangle_\psi \left( w^f_{j+1}, w^f_{j+\frac{1}{2}} \right) \\
& \leq \frac{1}{\eta_f} \cdot \triangle_\psi \left( w_f^{\ast}, w^f_{\tau +1} \right) + \frac{1}{\eta_f} \cdot \frac{T}{M} \cdot \tau \cdot \frac{\left( \eta_f \cdot G \right)^2 }{2} \\
& \leq \frac{1}{\eta_f} \cdot B^2 + \frac{T}{M} \cdot \tau \cdot \frac{ \eta_f \cdot G^2 }{2} &
\end{align*}%
\endgroup

For $w^s$: 
\begingroup
\allowdisplaybreaks
\begin{align*}
& \sum_{j=1}^{\frac{T}{M} \cdot \left( M - \tau \right)} \left\langle w^s_j - w_s^{\ast}, g_j \right\rangle \\
& = \sum_{j=1}^{\frac{T}{M} \cdot \left( M - \tau \right)} \left\langle w^s_j - w_s^{\ast}, \nabla f_{T_2 \left( j \right)} \left( w_j^s \right) \right\rangle \\
& = \sum_{j=1}^{\frac{T}{M} \cdot \left( M - \tau \right)} \frac{1}{\eta} \cdot \left\langle w^s_j - w_s^{\ast}, \left( \nabla \psi \left( w^s_j \right) - \nabla \psi \left( w^s_{j + \frac{1}{2}} \right) \right) \right\rangle \\
& = \sum_{j=1}^{\frac{T}{M} \cdot \left( M - \tau \right)} \frac{1}{\eta} \cdot \left( \triangle_\psi \left( w_s^{\ast}, w^s_j \right) + \triangle_\psi \left( w^s_j, w^s_{j + \frac{1}{2}} \right) - \triangle_\psi \left( w_s^{\ast}, w^s_{j + \frac{1}{2}} \right) \right) \\
& \leq \sum_{j=1}^{\frac{T}{M} \cdot \left( M - \tau \right)} \frac{1}{\eta} \cdot \left( \triangle_\psi \left( w_s^{\ast}, w^s_j \right) + \triangle_\psi \left( w^s_j, w^s_{j + \frac{1}{2}} \right) - \triangle_\psi \left( w_s^{\ast}, w^s_{j+1} \right)  - \triangle_\psi \left( w^s_{j+1}, w^s_{j+\frac{1}{2}} \right) \right) \\
& \leq \frac{1}{\eta} \cdot \sum_{j=1}^{\frac{T}{M} \cdot \left( M - \tau \right)} \triangle_\psi \left( w_s^{\ast}, w^s_j \right) - \triangle_\psi \left( w_s^{\ast}, w^s_{j+1} \right) + \frac{1}{\eta} \cdot \triangle_\psi \left( w^s_j, w^s_{j + \frac{1}{2}} \right) - \triangle_\psi \left( w^s_{j+1}, w^s_{j+\frac{1}{2}} \right) \\
& = \frac{1}{\eta} \cdot \triangle_\psi \left( w_s^{\ast}, w^s_{1} \right) - \triangle_\psi \left( w_s^{\ast}, w^s_{\left( \frac{T}{M} +1 \right) \cdot \tau} \right) + \frac{1}{\eta} \cdot \sum_{j=1}^{\frac{T}{M} \cdot \left( M - \tau \right)} \triangle_\psi \left( w^s_j, w^s_{j + \frac{1}{2}} \right) - \triangle_\psi \left( w^s_{j+1}, w^s_{j+\frac{1}{2}} \right) \\
& \leq \frac{1}{\eta_s} \cdot \triangle_\psi \left( w_s^{\ast}, w^s_{1} \right) + \frac{1}{\eta_s} \cdot \frac{T}{M} \cdot \left( M - \tau \right) \cdot \frac{\left( \eta_s \cdot G \right)^2 }{2} \\
& \leq \frac{1}{\eta_s} \cdot B^2 + \frac{T}{M} \cdot \left( M - \tau \right) \cdot \frac{ \eta_s \cdot G^2 }{2} &
\end{align*}%
\endgroup

We are after bounding the regret, which in itself is upper bounded by the sum of the regret accumulated by each sub-algorithm, considering iterations in the first $\tau$ and last $M-\tau$ per block separately, as mentioned above.
Using the convexity of $f_t$ for all $t$, we bound these terms:
\begin{align*}
 \mathbb{E} & \left[ \sum_{i=0}^{\frac{T}{M}-1} \sum_{t=M \cdot i +1}^{M \cdot i +\tau} f_t \left( w_t \right) -  f_t \left( w_f^{\ast} \right) +\sum_{i=0}^{\frac{T}{M}-1} \sum_{t=M \cdot i + \tau +1}^{M \cdot \left(i+1 \right)} f_t \left( w_t \right) -  f_t \left( w_s^{\ast} \right) \right] \\
\leq \mathbb{E} & \left[ \sum_{i=0}^{\frac{T}{M}-1} \sum_{t=M \cdot i +1}^{M \cdot i +\tau} \left\langle  w_t - w_f^\ast, \nabla f_t \left( w_t \right) \right\rangle +\sum_{i=0}^{\frac{T}{M}-1} \sum_{t=M \cdot i + \tau +1}^{M \cdot \left(i+1 \right)} \left\langle  w_t - w_s^\ast, \nabla f_t \left( w_t \right) \right\rangle \right] \\
 = \mathbb{E} & \left[ \sum_{j=1}^{\frac{T}{M} \cdot \tau} \left\langle  w_j^f - w_f^\ast, \nabla f_{T_1 \left( j \right)} \left( w_j^f \right) \right\rangle + \sum_{j=1}^{\frac{T}{M} \cdot \left( M - \tau \right) } \left\langle  w_j^s - w_s^\ast, \nabla f_{T_2 \left( j \right) + \tau} \left( w_j^s \right) \right\rangle \right] &
\end{align*}
In the last equality of the above derivation, we simply replace notations, writing the gradient $\nabla f_t \Par {w_t}$ in notation of $T_1$ and $T_2$. $T_1$ contains all time points in the first $\tau$ iterations of each block, and $T_2$ contains all time points in the first $M-\tau$ iterations of each block.

Note that what we have bounded so far is $\sum_{j=\tau + 1}^{( \frac{T}{M} +1 ) \cdot \tau} \langle w^f_j - w_f^{\ast}, \nabla f_{T_1 \left( j - \tau \right)} ( w_{j- \tau }^f ) \rangle$ for $w^f$ and $\sum_{j=1}^{\frac{T}{M} \cdot \left( M - \tau \right)} \left\langle w^s_j - w_s^{\ast}, \nabla f_{T_2 \left( j \right)} \left( w_j^s \right) \right\rangle$ for $w^s$, which are not the terms we need to bound in order to get a regret bound since they use the \emph{delayed} gradient, and so we need to take a few more steps in order to be able to bound the regret.

We begin with $w^f$:
\begingroup
\allowdisplaybreaks
\begin{align*}
\sum_{i=0}^{\frac{T}{M}-1} \sum_{t=M \cdot i +1}^{M \cdot i +\tau} \left\langle  w_t - w_f^\ast, \nabla f_t \left( w_t \right) \right\rangle & = \sum_{j=1}^{\frac{T}{M} \cdot \tau} \left\langle  w_j^f - w_f^\ast, \nabla f_{T_1 \left( j \right)} \left( w_j^f \right) \right\rangle \\
& = \sum_{j=1}^{\frac{T}{M} \cdot \tau} \left\langle  w_{j+\tau}^f - w_f^\ast, \nabla f_{T_1 \left( j \right)} \left( w_{j}^f \right) \right\rangle + \left\langle  w_{j}^f - w^f_{j + \tau}, \nabla f_{T_1 \left( j \right)} \left( w_{j}^f \right) \right\rangle \\
& = \sum_{j=\tau + 1}^{ \left(\frac{T}{M} + 1 \right) \cdot \tau} \left\langle  w_{j}^f - w_f^\ast, \nabla f_{T_1 \left( j - \tau \right)} \left( w_{j - \tau}^f \right) \right\rangle + \left\langle  w_{j - \tau}^f - w^f_{j}, \nabla f_{T_1 \left( j -\tau \right)} \left( w_{j - \tau}^f \right) \right\rangle \\
& \leq \frac{1}{\eta_f} \cdot B^2 + \frac{T}{M} \cdot \tau \cdot \frac{ \eta_f \cdot G^2 }{2} + \sum_{j=\tau + 1}^{ \left(\frac{T}{M} + 1 \right) \cdot \tau} \left\langle  w_{j - \tau}^f - w^f_{j}, \nabla f_{T_1 \left( j -\tau \right)} \left( w_{j - \tau}^f \right) \right\rangle \\
& \leq \frac{1}{\eta_f} \cdot B^2 + \frac{T}{M} \cdot \tau \cdot \frac{ \eta_f \cdot G^2 }{2} + \sum_{j=\tau + 1}^{ \left(\frac{T}{M} + 1 \right) \cdot \tau} \| w_{j - \tau}^f - w^f_{j} \| \cdot \| \nabla f_{T_1 \left( j -\tau \right)} \left( w_{j - \tau}^f \right) \| \\
& \leq \frac{1}{\eta_f} \cdot B^2 + \frac{T}{M} \cdot \tau \cdot \frac{ \eta_f \cdot G^2 }{2} + \sum_{j=\tau + 1}^{ \left(\frac{T}{M} + 1 \right) \cdot \tau} \sum_{i=1}^\tau \| w^f_{j-i} - w^f_{j-i+1} \| \cdot G &
\end{align*}%
\endgroup
The last term in the above derivation, is the sum of differences between consecutive predictors. This difference, is determined by the mirror map in use, the step size $\eta_f$, and the bound over the norm of the gradient used in the update stage of the algorithm, $G$. This is because every consecutive predictor is received by taking a gradient step from the previous predictor, in the dual space, with a step size $\eta_f$, and projecting back to the primal space by use of the bregman divergence with the specific mirror map in use. We denote the bound on this difference by $\Psi_{\Par{\eta_f, G}}$, i.e., $\forall j, j+1: \| w^f_j - w^f_{j+1} \| \leq \Psi_{\Par{\eta_f, G}}$.
Continuing our derivation, we have:
\begin{align*}
& \leq \frac{1}{\eta_f} \cdot B^2 + \frac{T}{M} \cdot \tau \cdot \frac{ \eta_f \cdot G^2 }{2} + \sum_{j=\tau + 1}^{ \left(\frac{T}{M} + 1 \right) \cdot \tau} \sum_{i=1}^\tau \Psi_{\Par{\eta_f, G}} \cdot G \\
& \leq \frac{1}{\eta_f} \cdot B^2 + \frac{T}{M} \cdot \tau \cdot \frac{ \eta_f \cdot G^2 }{2} + \frac{T}{M} \cdot \tau^2 \cdot \Psi_{\Par{\eta_f, G}} \cdot G &
\end{align*}
Since this upper bound does not depend on the permutation,and holds for every sequence, it holds also in expectation, i.e. 
\begin{align*}
\mathbb{E} \left[ \sum_{i=0}^{\frac{T}{M}-1} \sum_{t=M \cdot i +1}^{M \cdot i +\tau} f_t \left( w_t \right) -  f_t \left( w_f^{\ast} \right) \right] \leq \frac{1}{\eta_f} \cdot B^2 + \frac{T}{M} \cdot \tau \cdot \frac{ \eta_f \cdot G^2 }{2} + \frac{T}{M} \cdot \tau^2 \cdot \Psi_{\Par{\eta_f, G}} \cdot G 
\end{align*}

We now turn to $w^s$
\begingroup
\allowdisplaybreaks
\begin{align*}
& \sum_{i=0}^{\frac{T}{M}-1} \sum_{t=M \cdot i + \tau +1}^{M \cdot \left(i+1 \right)} f_t \left( w_t \right) -  f_t \left( w_s^{\ast} \right) \\
& \leq \sum_{i=0}^{\frac{T}{M}-1} \sum_{t=M \cdot i + \tau +1}^{M \cdot \left(i+1 \right)} \left\langle  w_t - w_s^\ast, \nabla f_t \left( w_t \right) \right\rangle \\
& = \sum_{j=1}^{\frac{T}{M} \cdot \left( M - \tau \right) } \left\langle  w_j^s - w_s^\ast, \nabla f_{T_2 \left( j \right) + \tau} \left( w_j^s \right) \right\rangle \\
& = \sum_{j=1}^{\frac{T}{M} \cdot \left( M - \tau \right)} \left\langle w^s_j - w_s^{\ast}, \nabla f_{T_2 \left( j \right)} \left( w_j^s \right) \right\rangle + \sum_{j=1}^{\frac{T}{M} \cdot \left( M - \tau \right)} \left\langle w^s_j - w_s^{\ast}, \nabla f_{T_2 \left( j \right) + \tau} \left( w_j^s \right) - \nabla f_{T_2 \left( j \right)} \left( w_j^s \right) \right\rangle \\
& \leq \frac{1}{\eta_s} \cdot B^2 + \frac{T}{M} \cdot \left( M - \tau \right) \cdot \frac{ \eta_s \cdot G^2 }{2} + \sum_{j=1}^{\frac{T}{M} \cdot \left( M - \tau \right)} \left\langle w^s_j - w_s^{\ast}, \nabla f_{T_2 \left( j \right) + \tau} \left( w_j^s \right) - \nabla f_{T_2 \left( j \right)} \left( w_j^s \right) \right\rangle &
\end{align*}%
\endgroup

We now look at the expression $\left\langle w^s_j - w_s^{\ast}, \nabla f_{T_2 \left( j \right) + \tau} \left( w_j^s \right) - \nabla f_{T_2 \left( j \right)} \left( w_j^s \right) \right\rangle$ for any $j$. 

We first notice that for any j, $w_j^s$ only depends on gradients of time points: $T_2\left( 1 \right),T_2\left( 2 \right),...,T_2\left( j-1 \right)$. \\
We also notice that given the functions received at these time points, i.e, given $f_{T_2\left( 1 \right)},f_{T_2\left( 2 \right)},...,f_{T_2\left( j-1 \right)}$, $w_j^s$ is no longer a random variable. \\ 
We have that for all $j$, $T_2\left( j \right)$ and $T_2\left( j \right) +\tau$ are both time points that are part of the same $M$-sized block.
Suppose we have observed $n$ functions of the block to which $T_2\left(j\right)$ and $T_2\left(j\right)+\tau$ belong. All of these $n$ functions are further in the past than both $T_2\left(j\right)$ and $T_2\left(j\right)+\tau$, because of the delay of size $\tau$. We have $M-n$ functions in the block that have not been observed yet, and since we performed a random permutation within each block, all remaining functions in the block have the same expected value. Formally, given $w_j^s$, the expected value of the current and delayed gradient are the same, since we have:
$\mathbb{E}[\nabla f_{T_2\left(j\right)+\tau}\left( w_j^s \right)|w_j^s] = \frac{1}{M-n} \cdot \sum_{i = 1}^{M-n} \nabla f_{T_2\left(j\right)+i} \left( w_j^s \right) = \mathbb{E}[\nabla f_{T_2\left(j\right)}\left( w_j^s \right)|w_j^s]$. As mentioned above, this stems from the random permutation we performed within the block - all $M-n$ remaining functions (that were not observed yet in this block) have an equal (uniform) probability of being in each location, and thus the expected value of the gradients is equal.
From the law of total expectation we have that 
\begin{align*}
\mathbb{E}[\nabla f_{T_2\left(j\right)+\tau}\left( w_j^s \right)]= \mathbb{E}[ \mathbb{E}[ \nabla f_{T_2\left(j\right)+\tau}\left( w_j^s \right)|w_j^s ]] = \mathbb{E}[ \mathbb{E}[ \nabla f_{T_2\left(j\right)}\left( w_j^s \right)|w_j^s ]] = \mathbb{E}[\nabla f_{T_2\left(j\right)}\left( w_j^s \right)]
\end{align*}
ans thus $\mathbb{E}[ \nabla f_{T_2 \left( j \right) + \tau} \left( w_j^s \right) - \nabla f_{T_2 \left( j \right)} \left( w_j^s \right) ] = 0 $. 

We get that $\mathbb{E}[ \left\langle w^s_j - w_s^{\ast}, \nabla f_{T_2 \left( j \right) + \tau} \left( w_j^s \right) - \nabla f_{T_2 \left( j \right)} \left( w_j^s \right) \right\rangle] = 0$

So we have that the upper bound on the expected regret of the time point in which we predict with $w^s$ is:
\begin{align*}
\mathbb{E} \left[ \sum_{i=0}^{\frac{T}{M}-1} \sum_{t=M \cdot i + \tau +1}^{M \cdot \left(i+1 \right)} f_t \left( w_t \right) -  f_t \left( w_s^{\ast} \right) \right] \leq \frac{1}{\eta_s} \cdot B^2 + \frac{T}{M} \cdot \left( M - \tau \right) \cdot \frac{ \eta_s \cdot G^2 }{2}
\end{align*}

Summing up the regret of the two sub-algorithms, we get:
\begin{align*}
\mathbb{E} \left[ \sum_{t=1}^{T} f_t \left( w_t \right) -  f_t \left( w^{\ast} \right) \right] 
& \leq \mathbb{E} \left[ \sum_{i=0}^{\frac{T}{M}-1} \sum_{t=M \cdot i +1}^{M \cdot i +\tau} f_t \left( w_t \right) -  f_t \left( w_f^{\ast} \right) +\sum_{i=0}^{\frac{T}{M}-1} \sum_{t=M \cdot i + \tau +1}^{M \cdot \left(i+1 \right)} f_t \left( w_t \right) -  f_t \left( w_s^{\ast} \right) \right] \\ 
& \leq \frac{B^2}{\eta_f} + \eta_f \cdot \frac{T \tau}{M} \cdot \frac{G^2}{2} + \frac{T \tau^2}{M} \cdot G \cdot \Psi_{\Par{\eta_f,G}} + \frac{B^2}{\eta_s} + \eta_s \cdot \frac{T \cdot \Par{M- \tau}}{M} \cdot \frac{G^2}{2} &
\end{align*}
which gives us the bound. \\
For $\Psi_{\Par{\eta_f,G}} \leq c \cdot \eta_f \cdot G$ where $c$ is some constant, choosing the step sizes, $\eta_f,\eta_s$ optimally:
\begin{align*}
\eta_f = \frac{B \cdot \sqrt{M}}{G \cdot \sqrt{ T \cdot \tau \cdot \left( \frac{1}{2} + c \cdot \tau \right)}}, \eta_s = \frac{B \cdot \sqrt{2 M}}{G \cdot \sqrt{T \cdot \left( M - \tau \right) }}
\end{align*}

we get the bound:
\begin{align*}
& \mathbb{E} \left[  \sum_{t=1}^{T} f_t \left( w_t \right) -  f_t \left( w^{\ast} \right) \right] \\ 
& = \sqrt{\frac{T \cdot \tau}{M}} \cdot B \cdot G \cdot \sqrt{\frac{1}{2} + c \cdot \tau} + \sqrt{\frac{T \cdot \tau}{M}} \cdot B \cdot G \cdot \frac{1}{\sqrt{ \frac{1}{2} + c \tau }} + \sqrt{\frac{T \cdot \tau}{M}} \cdot B \cdot G \cdot \frac{c \tau}{\sqrt{ \frac{1}{2} + c \tau }} \\
& + \sqrt{\frac{2 \cdot T \cdot \left( M -\tau \right)}{M}} \cdot B \cdot G  \\
& \leq c \cdot \sqrt{\frac{T \cdot \tau}{M}} \cdot B \cdot G \cdot \sqrt{\frac{1}{2} + c \cdot \tau} + \sqrt{\frac{2 \cdot T \cdot \left( M - \tau \right)}{M}} \cdot B \cdot G \\
& = \mathcal{O}\left( \sqrt{\frac{T \cdot \tau^2}{ M}} + \sqrt{\frac{T \cdot \left( M - \tau \right)}{ M}} \right) = \mathcal{O} \Par{\sqrt{T} \cdot \Par{\sqrt{\frac{\tau^2}{M}} + 1}} &
\end{align*}

\subsection{Lower Bound For Algorithms With No Permutation Power}
\label{app:lowerBoundAdvers}

\begin{theorem}
For every (possible randomized) algorithm $A$, there exists a choice of linear, $1$-Lipschitz functions over $[-1,1] \subset \mathbb{R}$, with $\tau$ a fixed size delay of feedback, such that the expected regret of $A$ after $T$ rounds (with respect to the algorithm's randomness), is
\begin{align*}
\mathbb{E} \left[ R_A \Par{T} \right] = \mathbb{E} \left[ \sum_{t=1}^{T} f_t \Par{w_t} - \sum_{t=1}^{T} f_t \Par{w^{\ast}} \right] = \Omega \Par{\sqrt{\tau T}} \text{, where } w^{\ast} = \underset{w \in \mathcal{W}}{\text{argmin }} \sum_{t=1}^T f_t \Par{ w} 
\end{align*}
\end{theorem}

\begin{proof}
First, we note that in order to show that for every algorithm, there exists a choice of loss functions by an oblivious adversary, such that the expected regret of the algorithm is bounded from below, it is enough to show that there exists a distribution over loss function sequences such that for any algorithm, the expected regret is bounded from below, where now expectation is taken over both the randomness of the algorithm and the randomness of the adversary. This is because if there exists such a distribution over loss function sequences, then for any algorithm, there exists some sequence of loss functions that can lead to a regret at least as high. To put it formally, if we mark $\underset{alg}{\mathbb{E}}$ the expectation over the randomness of the algorithm, and $\underset{f_1,...,f_T}{\mathbb{E}}$ the expectation over the randomness of the adversary, then:
\begin{gather*}
\exists \mbox{ a (randomized) adversary s.t. } \forall \mbox{ algorithm A, } \underset{f_1,...,f_T}{\mathbb{E}} \mbox{    } \underset{alg}{\mathbb{E}} \left[ R_A \Par{T} \right] > \Omega \Par{\sqrt{\tau T}} \rightarrow \\
\forall \mbox{ algorithm A, } \exists f_1,...,f_T \mbox{ s.t. } \underset{alg}{\mathbb{E}} \left[ R_A \Par{T} \right] > \Omega \Par{\sqrt{\tau T}}
\end{gather*}
Thus, we prove the first statement above, that immediately gives us the second statement which gives the lower bound.

We consider the setting where $\mathcal{W} = \left[ -1, 1 \right]$, and $\forall t\in [1,T]: f_t\left(w_t \right) = \alpha_t \cdot w_t$ where $\alpha_t \in \lbrace 1, -1 \rbrace$. We divide the $T$ rounds to blocks of size $\tau$. $\alpha_t$ is chosen in the following way: if $\alpha_t$ is the first $\alpha$ in the block, it is randomly picked, i.e, $\Pr \left( \alpha = \pm 1 \right) = \frac{1}{2}$. Following this random selection, the next $\tau - 1$ $\alpha$'s of the block will be identical to the first $\alpha$ in it, so that we now have a block of $\tau $ consecutive functions in which $\alpha$ is identical. 
We wish to lower bound the expected regret of any algorithm in this setting.

Consider a sequence of predictions by the algorithm $w_1,w_2,...,w_T$. Denote by $\alpha_{i,j}$ the j'th $\alpha$ in the i'th block, and similarly for $w_{i,j}, f_{i,j}$. We denote the entire sequence of $\alpha$'s by $\bar{\alpha}_{\left(1 \rightarrow T\right)}$, and the sequence of $\alpha$'s until time point $j$ in block $i$ by $\bar{\alpha}_{\left(1 \rightarrow i,j\right)}$.
Notice that $w_{i,j}$ is a function of the $\alpha$'s that arrive up until time point $i \cdot \tau  + j - \tau -1$. We denote these $\alpha$'s as $\bar{\alpha}_{\left(1 \rightarrow i,j - \tau -1 \right)}$. 

Then the expected sum of losses is:
\begingroup
\allowdisplaybreaks
\begin{align*}
\mathbb{E}\left[ \sum_{t=1}^T f_t \left( w_t \right) \right] & = \mathbb{E}\left[ \sum_{i=1}^{\frac{T}{\tau}} \sum_{j=1}^{\tau} f_{i,j} \left( w_{i,j} \right) \right] \\
& = \sum_{i=1}^{\frac{T}{\tau}} \sum_{j=1}^{\tau} \mathbb{E}\left[ f_{i,j} \left( w_{i,j} \right) \right] \\
& = \sum_{i=1}^{\frac{T}{\tau}} \sum_{j=1}^{\tau} \mathbb{E}_{\bar{\alpha}_{\left(1 \rightarrow T\right)}}\left[ \alpha_{i,j} \cdot w_{i,j} \right] \\
& = \sum_{i=1}^{\frac{T}{\tau}} \sum_{j=1}^{\tau} \mathbb{E}_{\bar{\alpha}_{\left(1 \rightarrow i,j- \tau -1 \right)}} \left[ \mathbb{E}_{\bar{\alpha}_{\left(i,j- \tau \rightarrow T \right)}} \left[ \alpha_{i,j} \cdot w_{i,j} | \bar{\alpha}_{\left(1 \rightarrow i,j - \tau -1 \right)} \right] \right] \\
& = \sum_{i=1}^{\frac{T}{\tau}} \sum_{j=1}^{\tau} \mathbb{E}_{\bar{\alpha}_{\left(1 \rightarrow i,j- \tau -1 \right)}} \left[ w_{i,j} \cdot \mathbb{E}_{\bar{\alpha}_{\left(i,j- \tau \rightarrow T \right)}} \left[ \alpha_{i,j} | \bar{\alpha}_{\left(1 \rightarrow i,j - \tau -1 \right)} \right] \right] \\
& = \sum_{i=1}^{\frac{T}{\tau}} \sum_{j=1}^{\tau} \mathbb{E}_{\bar{\alpha}_{\left(1 \rightarrow i,j- \tau -1 \right)}} \left[ w_{i,j} \cdot \mathbb{E}_{ \bar{\alpha}_{\left(i,1 \rightarrow i,j \right)}} \left[ \alpha_{i,j} | \bar{\alpha}_{\left(1 \rightarrow i,j - \tau -1 \right)} \right] \right] \\
& = \sum_{i=1}^{\frac{T}{\tau}} \sum_{j=1}^{\tau} \mathbb{E}_{\bar{\alpha}_{\left(1 \rightarrow i,j- \tau -1 \right)}} \left[ w_{i,j} \cdot \mathbb{E}_{ \alpha_{i,1}} \left[ \alpha_{i,1} \right] \right] \\
& = \sum_{i=1}^{\frac{T}{\tau}} \sum_{j=1}^{\tau} \mathbb{E}_{\bar{\alpha}_{\left(1 \rightarrow i,j- \tau -1 \right)}} \left[ w_{i,j} \cdot \left( \frac{1}{2} \cdot 1 + \frac{1}{2} \cdot \left(-1 \right) \right) \right] = 0 &
\end{align*}%
\endgroup

The last equality is true because every first $\alpha$ in any block has probability $\frac{1}{2}$ to be either $+1$ or $-1$.

We now continue to the expected sum of losses for the optimal choice of $w^{\ast} = argmin_{w \in \mathcal{W}} \left( \sum_{t=1}^T f_t \left( w \right) \right)$. Note that in this setting, $w^\ast \in \lbrace +1,-1 \rbrace$ and is with opposite sign to the majority of $\alpha$'s in the sequence.

\begin{align*}
\mathbb{E}\left[ \sum_{t=1}^T f_t \left( w^{\ast} \right) \right] & = \mathbb{E}\left[ \sum_{i=1}^{\frac{T}{\tau}} \sum_{j=1}^{\tau} f_{i,j} \left( w^{\ast} \right) \right] = \mathbb{E}\left[ \sum_{i=1}^{\frac{T}{\tau}} \sum_{j=1}^{\tau} \alpha_{i,j} \cdot  w^{\ast} \right] \\
& = \mathbb{E}\left[ \sum_{i=1}^{\frac{T}{\tau}} \tau \cdot \alpha_{i,1} \cdot  w^{\ast} \right] = \tau \cdot \mathbb{E}\left[ \sum_{i=1}^{\frac{T}{\tau}} \alpha_{i,1} \cdot  w^{\ast} \right] \\
& = - \tau \cdot \mathbb{E}\left[ |  \sum_{i=1}^{\frac{T}{\tau}} \alpha_{i,1} | \right]
\end{align*}

Using Khintchine inequality we have that:
\begin{align*}
-\tau \cdot \mathbb{E}\left[ |  \sum_{i=1}^{\frac{T}{\tau}} \alpha_{i,1} \cdot 1 | \right] \leq -  \tau \cdot C \cdot \sqrt{\left( \sum_{i=1}^{\frac{T}{\tau}} 1^2 \right) } = - \tau  \cdot C \cdot \sqrt{\frac{T}{\tau}} =  - \Omega \left( \sqrt{ \tau  \cdot T} \right)
\end{align*}

where $C$ is some constant.

Thus we get that for a sequence of length $T$ the expected regret is:
\begin{align*}
\mathbb{E}\left[ \sum_{t=1}^T f_t \left( w_t \right) \right] - \mathbb{E}\left[ \sum_{t=1}^T f_t \left( w^{\ast} \right) \right] = \Omega \left( \sqrt{\tau \cdot T} \right)
\end{align*}
\end{proof}
\subsection{Proof of Theorem~\ref{lemma:lowerBound}}
\label{app:lowerBoundProof}
\begin{proof}
First, we note that to show that for every algorithm, there exists a choice of loss functions by an oblivious adversary, such that the expected regret of the algorithm is bounded from below, it is enough to show that there exists a distribution over loss function sequences such that for any algorithm, the expected regret is bounded from below, where now expectation is taken over both the randomness of the algorithm and the randomness of the adversary. This is because if there exists such a distribution over loss function sequences, then for any algorithm, there exists some sequence of loss functions that can lead to a regret at least as high. To put it formally, if we mark $\underset{alg}{\mathbb{E}}$ the expectation over the randomness of the algorithm, and $\underset{f_1,...,f_T}{\mathbb{E}}$ the expectation over the randomness of the adversary, then:
\begin{gather*}
\exists \mbox{ a (randomized) adversary s.t. } \forall \mbox{ algorithm A, } \underset{f_1,...,f_T}{\mathbb{E}} \mbox{    } \underset{alg}{\mathbb{E}} \left[ R_A \Par{T} \right] > \Omega \Par{\sqrt{\tau T}} \rightarrow \\
\forall \mbox{ algorithm A, } \exists f_1,...,f_T \mbox{ s.t. } \underset{alg}{\mathbb{E}} \left[ R_A \Par{T} \right] > \Omega \Par{\sqrt{\tau T}}
\end{gather*}
Thus, we prove the first statement above, that immediately gives us the second statement which is indeed our lower bound.

We consider the setting where $\mathcal{W} = \left[ -1, 1 \right]$, and $\forall t\in [1,T]: f_t\left(w_t \right) = \alpha_t \cdot w_t$ where $\alpha_t \in \lbrace 1, -1 \rbrace$. We start by constructing our sequence of $\alpha$'s. We divide the $T$ iterations to blocks of size $\frac{\tau}{3}$. In each block, all $\alpha$'s are identical, and are chosen to be $+1$ or $-1$ w.p. $\frac{1}{2}$. This choice gives us blocks of $\frac{\tau}{3}$ consecutive functions in which $\alpha$ is identical within each block. 
Let $M$ be a permutation window of size smaller than $\frac{\tau}{3}$.
We notice first that since $M < \frac{\tau}{3}$ and the sequence of $\alpha$'s is organized in blocks of size $\frac{\tau}{3}$, then even after permutation, the time difference between the first and last time we encounter an $\alpha$ is $\leq \tau$, which means we will not get the feedback from the first time we encountered this $\alpha$ before encountering the next one, and we will not be able to use it for correctly predicting $\alpha$'s of this (original) block that arrive later. This is the main idea that stands in the basis of this lower bound. 

Formally, consider a sequence of $w_1,w_2,...,w_T$ chosen by the algorithm. Denote by $\alpha_{i,j}$ the j'th $\alpha$ in the i'th block, and similarly for $w_{i,j}, f_{i,j}$. We denote the entire sequence of $\alpha$'s by $\bar{\alpha}_{\left(1 \rightarrow T\right)}$, and the sequence of $\alpha$'s until time point $j$ in block $i$ by $\bar{\alpha}_{\left(1 \rightarrow i,j\right)}$.
For simplicity we will denote $\beta_t$ as the $\alpha$ that was presented at time $t$, after permutation, i.e. $\beta_t := \alpha_{\sigma^-1 \Par(t)}$.
Notice that $w_{i,j}$ is a function of the $\beta$'s that arrive up until time point $i \cdot \left( \frac{\tau}{3} \right) + j - \tau -1$. We denote these $\beta$'s as $\bar{\beta}_{\left(1 \rightarrow i,j - \tau -1 \right)}$. I.e $w_{i,j} = g \left( \bar{\beta}_{\left(1 \rightarrow i,j - \tau -1 \right)} \right)$ where $g$ is some function.

Going back to our main idea of the construction, we can put it in this new terminology- since the delay is $\tau$ and the permutation window is $M < \frac{\tau}{3}$, for any $i,j$, the first time we encountered $\alpha_{\sigma^{-1}\left( i,j \right)}$ is less than $\tau$ iterations ago, and thus, $\beta_{i,j}$ is independent of $\bar{\beta}_{\left( 1 \rightarrow i,j - \tau -1 \right)}$, while $w_{i,j}$ is a function of it: $w_{i,j} = g \left( \bar{\beta}_{\left( 1 \rightarrow i,j - \tau -1 \right)} \right)$. 

With this in hand, we look at the sum of losses of the predictions of the algorithm, $w_1,w_2,...,w_T$:
\begingroup
\allowdisplaybreaks
\begin{align*}
\mathbb{E}\left[ \sum_{t=1}^T f_t \left( w_t \right) \right] & = \mathbb{E}\left[ \sum_{i=1}^{\nicefrac{T}{\frac{\tau}{3}}} \sum_{j=1}^{\frac{\tau}{3}} f_{i,j} \left( w_{i,j} \right) \right] \\
& = \sum_{i=1}^{\nicefrac{T}{\frac{\tau}{3}}} \sum_{j=1}^{\frac{\tau}{3}} \mathbb{E}\left[ f_{i,j} \left( w_{i,j} \right) \right] \\ 
& = \sum_{i=1}^{\nicefrac{T}{\frac{\tau}{3}}} \sum_{j=1}^{\frac{\tau}{3}} \mathbb{E}_{\bar{\beta}_{\left(1 \rightarrow T\right)}}\left[ \beta_{i,j} \cdot w_{i,j} \right] \\
& = \sum_{i=1}^{\nicefrac{T}{\frac{\tau}{3}}} \sum_{j=1}^{\frac{\tau}{3}} \mathbb{E}_{\bar{\beta}_{\left(1 \rightarrow i,j- \tau -1 \right)}} \left[ \mathbb{E}_{\bar{\beta}_{\left(i,j- \tau \rightarrow T \right)}} \left[ \beta_{i,j} \cdot w_{i,j} | \bar{\beta}_{\left(1 \rightarrow i,j - \tau -1 \right)} \right] \right] \\
& = \sum_{i=1}^{\nicefrac{T}{\frac{\tau}{3}}} \sum_{j=1}^{\frac{\tau}{3}} \mathbb{E}_{\bar{\beta}_{\left(1 \rightarrow i,j- \tau -1 \right)}} \left[ w_{i,j} \cdot \mathbb{E}_{\bar{\beta}_{\left(i,j- \tau \rightarrow T \right)}} \left[ \beta_{i,j} | \bar{\beta}_{\left(1 \rightarrow i,j - \tau -1 \right)} \right] \right] \\
& = \sum_{i=1}^{\nicefrac{T}{\frac{\tau}{3}}} \sum_{j=1}^{\frac{\tau}{3}} \mathbb{E}_{\bar{\beta}_{\left(1 \rightarrow i,j- \tau -1\right)}} \left[ w_{i,j} \cdot \mathbb{E}_{ \bar{\beta}_{\left(i,j- \tau \rightarrow T \right)}} \left[ \alpha_{\sigma^{-1}\left( i,j \right)} \right] \right] \\
& = \sum_{i=1}^{\nicefrac{T}{\frac{\tau}{3}}} \sum_{j=1}^{\frac{\tau}{3}} \mathbb{E}_{\bar{\beta}_{\left(1 \rightarrow i,j- \tau -1\right)}} \left[ w_{i,j} \cdot \left( \frac{1}{2} \cdot 1 + \frac{1}{2} \cdot \left(-1 \right) \right) \right] = 0 &
\end{align*}%
\endgroup

where the last equality stems from the fact that $\beta_{i,j} = \alpha_{\sigma^{-1}\left(i,j\right)}$ is equal to the expected value of the first time we encountered the $\alpha$ that corresponds to $\alpha_{\sigma^{-1}\left(i,j\right)}$, i.e, the first $\alpha$ that came from the same block of $\alpha_{\sigma^{-1}\left(i,j\right)}$. This expectation is 0 since we choose $\alpha=1$ or $\alpha=-1$ with probability $\frac{1}{2}$ for each block.

We now continue to the expected sum of losses for the optimal choice of $w^{\ast} = argmin_{w \in \mathcal{W}} \left( \sum_{t=1}^T f_t \left( w \right) \right)$. Note that after permutation, the expected sum of losses of the optimal $w$ remains the same since it is best predictor over the entire sequence, and so for simplicity we look at the sequence of $\alpha$'s as it is chosen initially. Also, in this setting, $w^\ast \in \lbrace +1,-1 \rbrace$ and is with opposite sign to the majority of $\alpha$'s in the sequence.

\begin{align*}
\mathbb{E}\left[ \sum_{t=1}^T f_t \left( w^{\ast} \right) \right] & = \mathbb{E}\left[ \sum_{i=1}^{\nicefrac{T}{\frac{\tau}{3}}} \sum_{j=1}^{\frac{\tau}{3}} f_{i,j} \left( w^{\ast} \right) \right] = \mathbb{E}\left[ \sum_{i=1}^{\nicefrac{T}{\frac{\tau}{3}}} \sum_{j=1}^{\frac{\tau}{3}} \alpha_{i,j} \cdot  w^{\ast} \right] \\
& = \mathbb{E}\left[ \sum_{i=1}^{\nicefrac{T}{\frac{\tau}{3}}} \frac{\tau}{3} \cdot \alpha_{i,1} \cdot  w^{\ast} \right] = 
\frac{\tau}{3} \cdot \mathbb{E}\left[ \sum_{i=1}^{\nicefrac{T}{\frac{\tau}{3}}} \alpha_{i,1} \cdot  w^{\ast} \right] \\
& = -  \frac{\tau}{3} \cdot \mathbb{E}\left[ |  \sum_{i=1}^{\nicefrac{T}{\frac{\tau}{3}}} \alpha_{i,1} | \right]
\end{align*}

Using Khintchine inequality we have that:
\begin{align*}
-\frac{\tau}{3} \cdot \mathbb{E}\left[ |  \sum_{i=1}^{\nicefrac{T}{\frac{\tau}{3}}} \alpha_{i,1} \cdot 1 | \right] & \leq - \frac{\tau}{3} \cdot C \cdot \sqrt{\left( \sum_{i=1}^{\nicefrac{T}{\frac{\tau}{3}}} 1^2 \right) } = - \frac{\tau}{3} \cdot C \cdot \sqrt{\frac{T}{\frac{\tau}{3}}} \\ 
& = - \Omega \left( \sqrt{ \frac{\tau}{3} \cdot T} \right) = - \Omega \left( \sqrt{\tau \cdot T} \right)
\end{align*}

where $C$ is some constant.

Thus we get that overall expected regret for any algorithm with permutation power $M < \frac{\tau}{3}$ is:
\begin{align*}
\mathbb{E}\left[ \sum_{t=1}^T f_t \left( w_t \right) \right] - \mathbb{E}\left[ \sum_{t=1}^T f_t \left( w^{\ast} \right) \right] = \Omega \left( \sqrt{\tau \cdot T} \right)
\end{align*}

as in the adversarial case.

\end{proof}

\end{document}